\pdfoutput=1
\PassOptionsToPackage{table}{xcolor}
\documentclass{article} % For LaTeX2e
\usepackage[utf8]{inputenc}
\usepackage{iclr2027_conference,times}

\usepackage{amsmath,amsfonts,bm}

\def\eqref#1{equation~\ref{#1}}
\def\1{\bm{1}}

\DeclareMathAlphabet{\mathsfit}{\encodingdefault}{\sfdefault}{m}{sl}
\SetMathAlphabet{\mathsfit}{bold}{\encodingdefault}{\sfdefault}{bx}{n}

\newcommand{\E}{\mathbb{E}}

\newcommand{\R}{\mathbb{R}}

\usepackage{url}
\usepackage{enumitem}
\usepackage{amssymb,amsthm,mathtools}
\usepackage{booktabs}
\usepackage{xcolor}
\usepackage{tikz}
\usepackage{graphicx}
\usepackage{tabularx}
\usepackage{flafter}
\usepackage{float}
\definecolor{hmA}{HTML}{F3F8FE}\definecolor{hmB}{HTML}{DCEAFC}\definecolor{hmC}{HTML}{CDE2FB}
\definecolor{hmD}{HTML}{B7D3F6}\definecolor{hmE}{HTML}{9EC5F4}
\definecolor{oursrow}{HTML}{ECEFFB}
\newcolumntype{L}{>{\raggedright\arraybackslash}X}
\newcolumntype{Y}[1]{>{\centering\arraybackslash}p{#1}}
\usetikzlibrary{arrows.meta,positioning,fit,backgrounds}

\usepackage{hyperref}
\hypersetup{colorlinks=true,linkcolor=blue!55!black,citecolor=blue!55!black,urlcolor=blue!55!black}

\newtheorem{definition}{Definition}
\theoremstyle{definition}   % upright: Assumption 1 is long, italic hurts legibility there
\newtheorem{assumption}{Assumption}
\theoremstyle{plain}
\newtheorem{theorem}{Theorem}
\newtheorem{corollary}{Corollary}
\newtheorem{proposition}{Proposition}
\theoremstyle{remark}
\newtheorem{remark}{Remark}
\newcommand{\rot}[1]{\makebox[0pt][l]{\hspace{-0.9em}\rotatebox{55}{#1}}}
\newcommand{\cB}{\mathcal{B}}\newcommand{\cH}{\mathcal{H}}
\newcommand{\cR}{\mathcal{R}}\newcommand{\cS}{\mathcal{S}}\newcommand{\cX}{\mathcal{X}}
\newcommand{\cY}{\mathcal{Y}}\newcommand{\cZ}{\mathcal{Z}}
\newcommand{\Prb}{\mathbb{P}}\newcommand{\ind}{\mathbf{1}}
\newcommand{\W}{\mathsf{W}}\newcommand{\Ctrl}{\mathfrak{C}}
\newcommand{\cK}{\mathcal{K}}
\newcommand{\eout}{\epsilon_{\mathrm{outside}}}
\newcommand{\eoutn}[1]{\epsilon_{\mathrm{outside},#1}}

\newcommand{\efairn}[1]{\epsilon_{\mathrm{fair},#1}}
\newcommand{\heoutn}[1]{\widehat{\epsilon}_{\mathrm{outside},#1}}
\providecommand{\cA}{\mathcal{A}}\providecommand{\cC}{\mathcal{C}}\providecommand{\cD}{\mathcal{D}}\providecommand{\cF}{\mathcal{F}}
\providecommand{\cI}{\mathcal{I}}\providecommand{\cL}{\mathcal{L}}\providecommand{\cO}{\mathcal{O}}\providecommand{\cU}{\mathcal{U}}
\providecommand{\supp}{\operatorname{supp}}\providecommand{\push}{\mathbin{\#}}
\title{What Can a Leaderboard Certify? Compositional Controllability for Fair Evaluation and Training of Biomedical Literature-Review Agents}

\author{Zhaowei Han$^{*}$\quad Xiang Zhang$^{*}$\quad Lingxiao Guan$^{*}$\quad Danqi Hu \\
\textbf{Kai Liu\quad Kevin Chang\quad Jie Liu} \\
University of Michigan, Ann Arbor \\
{\small\texttt{\{rickyhan,xiangzhg,lxguan,danqihu,kailiua,kvchang,drjieliu\}@umich.edu}}
}

\iclrfinalcopy
\begin{document}

\maketitle
\lhead{Preprint}
{\renewcommand{\thefootnote}{\fnsymbol{footnote}}\footnotetext[1]{Equal contribution.}}
\begin{abstract}
Leaderboards rank long-horizon agents by their final outputs. Yet a higher score alone does not establish whether two systems are comparable or which stage accounts for the difference. Unequal evidence, inputs, or budgets can affect scores, and statistical corrections do not remove this mismatch. We introduce \emph{compositional controllability} to address these questions. A comparison window covers one stage, several stages, or the whole agent. Our central result bounds the gap between observed and controlled score differences using only nuisance outside the window. This yields an admissibility test applied before scores are inspected. Inadmissible comparisons are refused. For admissible pairs, an ordering is certified only when the score gap exceeds the combined sampling and nuisance radii; otherwise, it remains undecided. These decisions give each system a rank interval. We introduce \emph{BioLitBench}, a benchmark of $2{,}042$ biomedical articles represented as structured claim graphs. We use it to compare published pipelines, commercial agents, and our own agent under a shared reference list or a shared frozen corpus. Among seven published pipelines, a conventional statistical analysis declares a winner in $14$ of $21$ pairwise comparisons. Yet the top-ranked system alone received the target review's bibliography, giving it access to evidence the others were not supplied. To isolate pipeline performance, our test requires matched inputs and a fixed backbone model. It refuses $11$ of the $21$ comparisons, including every comparison involving the top-ranked system. Seven of the $14$ conventional conclusions fall within these refused pairs. The same comparison windows support stage-level training. We train \emph{SCRIBE} (\textbf{S}tage-\textbf{C}ertified \textbf{R}eview-writing agent with \textbf{I}nterface-\textbf{B}ounded \textbf{E}vidence) on Qwen3.8-27B using rewards measured at each stage's exit. Under matched evidence, SCRIBE achieves a certified rank interval of $[1,2]$, with certified advantages over all evaluated published pipelines and the evaluated Claude and OpenAI agents. Under same pool, SCRIBE matches the strongest published retriever and is certified above three published pipelines. The code is available at \url{https://github.com/shawnzhg/LitReview-SCRIBE}
\end{abstract}
\section{Introduction}

Leaderboards are a common way to evaluate long-horizon agents. A literature-review agent searches for papers, screens them, extracts evidence, develops claims, and writes a report. A benchmark then scores the final report. The resulting table tells us which output scored higher. But it leaves two central questions unanswered:
\begin{center}
    \emph{Are the systems comparable, and where does the difference come from?}
\end{center}

The first question concerns the evaluation design. Systems may receive different evidence, inputs, or budgets. Their scores then reflect both these conditions and the capabilities we want to compare. More tasks, clustering, and multiplicity correction do not remove this mismatch. The second question concerns attribution. Each stage changes what the next stage receives. A better report may come from better retrieval, better writing, or both. To compare writing stages, we must account for differences in their upstream inputs. A final score alone cannot separate these effects.

We introduce \emph{compositional controllability} to address both questions. We define a \emph{comparison window} by its entry and exit frontiers. A window can cover one stage, several stages, or the whole agent. A common interface makes these windows comparable across different systems. For each window and readout, we define a \emph{controlled contrast}: the score difference from a common entry under matched observation and protocol conditions. We then bound the gap between this contrast and the difference observed in an evaluation.

Our bound depends only on nuisance \emph{outside} the window: entry mismatch, differences in exit observation, and protocol imbalance. It places no restriction on what happens inside. This distinction matters. Retrieval differences are part of the capability being tested in a whole-agent comparison. In a comparison of writing stages, those same differences affect the inputs and must be accounted for.

The bound gives an admissibility test that is applied before scores are inspected. We \emph{refuse} a comparison if its outside nuisance cannot be bounded within a declared tolerance. For an admissible pair, we \emph{certify} an ordering only if the score gap exceeds the sum of the sampling and nuisance radii. Otherwise, the result is \emph{undecided}. The resulting ranking is partial. Each system receives a rank interval consistent with the certified relations. A leaderboard can therefore support an ordering only when external mismatch is controlled or bounded. Some positions may remain unresolved.

A systematic-review evaluation illustrates the problem. A production assistant reports abstract-screening sensitivity of $96.9\%$ on $108$ Cochrane reviews~\citep{elicit2025slr}. It describes this result as exceeding single-reviewer and approaching dual-reviewer human performance. The human estimates are $86.6\%$ and $97.5\%$. They come from a separate trial in which crowd screeners classified $2{,}000$ abstracts from two other reviews~\citep{gartlehner2020single}. These evaluations use different reviews, abstracts, and reference standards. No bound on this mismatch is stated, and no sampling radius is reported for the cross-study contrast. Two reviews also provide little basis for estimating between-review variation. Even a statistically significant gap would not resolve these differences in design. Yet the assistant's evaluation already tests individual stages against gold inputs. This is the controlled-window design we formalize. What our framework adds is an explicit admissibility test and an uncertainty radius.

We study these questions in biomedical literature review. Published human reviews provide content references against which generated reports can be scored claim by claim. Our benchmark, \emph{BioLitBench}, contains $2{,}042$ biomedical articles, most of them reviews. It represents them as graphs with $246{,}102$ claims, $832{,}487$ typed relations between claims, $103{,}242$ hard negatives, and section hierarchies. Reviews on the same topic form set-valued references that calibrate the comparison tolerances.

We use BioLitBench to build a \emph{certified leaderboard} of published pipelines, commercial agents, and our own agent. We evaluate two entry conditions. In the first, every system receives the target review's reference list. In the second, every system searches the same frozen corpus. The certificates concern agreement with human reference content. They do not establish how a system used its evidence. An evidence-ablation control demonstrates this limitation of the content readouts.

Certification changes which conclusions the evaluation supports. Among seven published pipelines, a cluster-robust, multiplicity-corrected analysis declares a winner in $14$ of $21$ pairwise comparisons. However, the top-ranked system alone received the target review's bibliography. Its score advantage therefore does not isolate pipeline quality from the benefit of supplied evidence. To compare pipelines under matched conditions, our admissibility test checks their inputs and requires a fixed backbone model. It refuses $11$ of the $21$ pairs, including every comparison involving the top-ranked system. Seven of the $14$ statistically significant orderings fall within these refused pairs. These score differences may be statistically reliable, but they do not establish which pipeline is better under matched conditions.

The same windows also support stage-level training. A sealed entry artifact removes entry mismatch. With observation and protocol conditions matched, an exit readout can evaluate the window's contribution. This allows us to reward individual stages rather than only complete trajectories. We build \emph{SCRIBE} (\textbf{S}tage-\textbf{C}ertified \textbf{R}eview-writing agent with \textbf{I}nterface-\textbf{B}ounded \textbf{E}vidence) around these windows. We train it with KV-Skill~\citep{kvskill2026} on a $27$B open model. Training uses the BioLitBench training split, which is disjoint from the $50$ evaluation tasks. Under matched evidence, SCRIBE has a certified rank interval of $[1,2]$. It has certified advantages over every evaluated published pipeline and the evaluated Claude and OpenAI agents. Its comparison with its own untrained harness remains undecided.

Our contributions are:
\begin{enumerate}[leftmargin=1.4em,itemsep=1pt]
    \item \textbf{Compositional controllability.}
    We define comparison windows over one stage, several stages, or the whole agent, and bound attribution error using only nuisance outside each window. This yields admissibility tests and certified partial rankings. We keep comparability separate from equivalence and substitution, and test the bound on the published pipelines themselves (Section~\ref{sec:method}; tests in Section~\ref{sec:theory-results} and Appendix~\ref{app:theory-checks}).

    \item \textbf{BioLitBench.}
    We introduce a graph-structured benchmark of $2{,}042$ biomedical articles with claims, typed relations, hard negatives, and section hierarchies. Same-topic reviews serve as set-valued references. A training split disjoint from the evaluation tasks lets the benchmark support both evaluation and training (Section~\ref{sec:benchmark}).

   \item \textbf{A certified leaderboard.}
We compare published pipelines, commercial agents, and SCRIBE under two entry conditions: a shared reference list or a shared frozen corpus. We report certified rank intervals and explicitly identify inadmissible comparisons. These certificates establish differences in agreement with human review content (Section~\ref{sec:leaderboard}).

    \item \textbf{SCRIBE.}
    We build a literature-review agent whose synthesis, planning, and writing stages each start from a sealed artifact. Each stage can then be rewarded at its own exit, rather than through the final report alone. We train SCRIBE with KV-Skill on a $27$B open model and also report SCRIBE-Luna, the same harness untrained on a frontier model. Under matched evidence, SCRIBE achieves a certified rank interval of $[1,2]$, above every evaluated published pipeline and the Claude and OpenAI agents.
\end{enumerate}
\section{Related work}
\label{sec:related}

Our machinery comes from probabilistic bisimulation
\citep{givan2003equivalence,ferns2004metrics,ferns2011bisimulation,gelada2019deepmdp} and causal
abstraction \citep{rubenstein2017causal,rischel2021compositional,felekis2024causal}, which ask when
one process may stand in for another. We ask the weaker question of when a contrast observed across
two \emph{separate} pipelines identifies the effect of replacing one stage, and
Theorem~\ref{thm:observable-attribution} answers it with a bound in the manner of partial
identification \citep{manski2003partial}. Stage-level credit assignment
\citep{lightman2024verify,wu2026optimas,zhang2025which} works inside one system, whereas we credit a
difference between two systems to the stage in which they differ.

Survey-writing and deep-research agents
\citep{DBLP:conf/nips/WangGYZZ0ZD0W0Z24,DBLP:conf/acl/YanFYX00Z25,DBLP:journals/corr/abs-2510-07733,DBLP:conf/ijcnlp/ChenYSLBGP25,DBLP:journals/corr/abs-2504-05732,DBLP:conf/aaai/GoLSTRC26,shao2026dr}
are benchmarked on their final reports \citep{su2026surge,du2026deepresearch}. \citet{zhu2026llm}
argue that what such a score supports depends on the declared boundary; a comparison window is such a
boundary, and the nuisance radius bounds what lies outside it. We adopt rigorous measurement practice
\citep{miller2024errorbars,bean2025measuring,kapoor2025agents,singh2025leaderboard}, but no test on the
scores reveals that a pair was never entitled to be compared. Confidence sets for ranks
\citep{klein2020joint,mogstad2024inference} give simultaneous inference on an ordering; our rank
intervals are of this kind, and what we add is the admissibility test that precedes them. Other work inspects a single run, through its logs
\citep{kirgis2026log} or with an agent that grades its trajectory \citep{zhuge2025agentjudge}, rather
than licensing a comparison between two systems. ADRA-Bank \citep{guo2025adra} sets stage inputs only
for backbone models, whereas published pipelines expose no entry but the task.

\section{Method: Compositional Controllability}
\label{sec:method}

Compositional controllability determines which comparisons an evaluation can support. A \emph{comparison window} specifies the stretch of execution being compared, a \emph{controlled contrast} defines the difference of interest, an \emph{attribution bound} limits the gap between the observed and controlled contrasts, and a \emph{certificate} combines this bound with sampling uncertainty. The same bound applies to one stage, several stages, or the whole agent.

An agent system $\cS$ induces a distribution over execution histories $h\in\cH$, since model sampling, tool calls and branching decisions all introduce randomness. An evaluation draws tasks from a distribution $\rho$, runs each system and scores its outputs, so all bounds below are relative to $\rho$ and to the states reached during evaluation. We measure differences between distributions using the Wasserstein distance $\W_d$ for a bounded metric $d$. A Markov kernel $Q$ is \emph{$L$-Lipschitz} if $\W_d(\nu Q,\nu' Q)\le L\,\W_d(\nu,\nu')$ for the input distributions reached by the benchmark, so that $L$ bounds how much $Q$ can amplify a difference in its inputs. Appendix~\ref{app:setting} gives the formal setting.

\subsection{A motivating example}
\label{sec:compare}

Consider two writing modules, $A$ and $B$, each receiving an evidence bundle with a fraction $e\in[0,1]$ of the required papers. Their reports score $f_A(e)=0.8e$ and $f_B(e)=0.5e$, so $A$ scores higher at any common coverage $e>0$, but the modules belong to pipelines with different retrieval stages.

\begin{center}
\small
\begin{tabular}{@{}lccc@{}}
\toprule
Module & Retrieved coverage & Pipeline score & Score at $e=0.5$ \\
\midrule
$A$ & $e_A=0.2$ & $0.16$ & $0.40$ \\
$B$ & $e_B=0.4$ & $0.20$ & $0.25$ \\
\midrule
$A-B$ & --- & $-0.04$ & $+0.15$ \\
\bottomrule
\end{tabular}
\end{center}

The pipeline scores favor $B$ by $0.04$, but $B$ receives more evidence. Given the same bundle with $e=0.5$, $A$ leads by $0.15$, so the pipeline comparison reverses the ordering of the writing modules, and because this mismatch is systematic, more tasks do not remove it.

\subsection{Comparison windows and readouts}

\paragraph{Windows.}
Different systems need not share code or internal representations, so we compare them through a \emph{semantic interface} that maps their states into a common space; $\cZ_M$ denotes its exit part for a window $M$. For a review agent, it records the evidence, the artifacts built from it, and the remaining budget.

A \emph{comparison window} $M=[F^{\mathrm{in}},F^{\mathrm{out}}]$ runs from an entry frontier to an exit frontier, each an antichain of histories that an execution visits at most once (Appendix~\ref{app:windows}). For system $X$, all computation inside the window is represented by a \emph{macro-kernel} $Q_X^M:F^{\mathrm{in}}\rightsquigarrow\cX_M\times F^{\mathrm{out}}$, where $\cX_M$ records the labeled trace, resources used and failures inside the window, and $F^{\mathrm{out}}$ holds the exit artifact.

A review agent works through a sequence of stages, from retrieval to writing. A \emph{stage window} covers one of these stages and begins at its input artifact, while a \emph{system window} covers the whole agent, from the task input to the final report. In our example, the window covers writing and begins at the evidence bundle. What lies inside the window may differ between the two systems, so a system window, which contains the backbone model, compares complete systems even across backbones.

\paragraph{Readouts.}
A benchmark evaluates a window through a readout $n$ with score function $s_n:\cY_n\to[0,1]$. The \emph{observation channel} $C_{X,n}:\cZ_M\rightsquigarrow\cY_n$ maps the window's exit state to the object being scored, including any downstream stages needed to obtain that object. When a retrieval window is scored on the final report, for instance, every later stage belongs to its channel.

Let $\mu_X$ be system $X$'s \emph{entry law}, the distribution of its inputs to the window. Its expected observed score is
\begin{equation}
\theta_{X,n}^{\mathrm{obs}}
=
V_n\big(\mu_X Q_X^M C_{X,n}\big),
\qquad
V_n(\lambda)
=
\int_{\cY_n}s_n(y)\,\lambda(dy).
\label{eq:observed-value}
\end{equation}

\subsection{Contrasts and the attribution bound}

A comparison between two pipelines may change both the window and its surroundings, since the systems may enter the window with different inputs or be scored through different channels. To isolate the window, fix a common entry law $\mu$ and a common channel $C_n$. The \emph{controlled contrast} and the \emph{observed contrast} are
\begin{equation}
\Delta_n^{\star}(\mu,C_n)=V_n(\mu Q_A^M C_n)-V_n(\mu Q_B^M C_n),
\qquad
\Delta_n^{\mathrm{obs}}=\theta_{A,n}^{\mathrm{obs}}-\theta_{B,n}^{\mathrm{obs}}.
\label{eq:controlled-contrast}
\end{equation}

The controlled contrast changes only the window, which makes it the quantity we want to estimate, but it cannot be measured when existing pipelines cannot be run from the same entry law or through the same channel. The observed contrast uses each pipeline's actual inputs and channel, and the two can have opposite signs, as in our example, where $\Delta^\star=+0.15$ but $\Delta^{\mathrm{obs}}=-0.04$.

We bound the gap between the observed and controlled contrasts under two conditions (Assumption~\ref{ass:main} in Appendix~\ref{app:windows}). First, all effects of the entry on the score must pass through the window. Second, differences outside the window must have bounded effects on the score.

In the bound, the entry mismatch $\epsilon_{\mathrm{in},X}$ is the Wasserstein distance between $\mu_X$ and $\mu$, and the channel mismatch $\epsilon_{\mathrm{cont},X,n}$ (for the continuation that forms the channel) is the largest distance between the outputs of $C_{X,n}$ and $C_n$ from the same reachable exit. The Lipschitz constants $L_{M_X}$ and $L_{C_n}$ bound how much the window and the common channel amplify input differences, and the slack $\xi_{M_X}$ allows for discontinuities such as truncation. Finally, $L_{s_n}$ is the score's Lipschitz constant, and $\epsilon_{\mathrm{protocol},n}$ bounds any remaining protocol imbalance on the score scale.

\begin{theorem}[Observable attribution bound]
\label{thm:observable-attribution}
Under Assumption~\ref{ass:main} and the boundary factorization of Appendix~\ref{app:windows}, $\big|\Delta_n^{\mathrm{obs}}-\Delta_n^{\star}(\mu,C_n)\big|\le\eoutn{n}$, where
\begin{equation}
\begin{aligned}
\eoutn{n}
:={}&
L_{s_n}\Big[
L_{C_n}\big(
L_{M_A}\epsilon_{\mathrm{in},A}
+
L_{M_B}\epsilon_{\mathrm{in},B}
+
\xi_{M_A}
+
\xi_{M_B}
\big)\\
&\qquad
+
\epsilon_{\mathrm{cont},A,n}
+
\epsilon_{\mathrm{cont},B,n}
\Big]
+
\epsilon_{\mathrm{protocol},n},
\end{aligned}
\label{eq:outside-bound}
\end{equation}
and the slack $\xi_{M_X}$ is charged only when $\mu_X\neq\mu$; it is $0$ when $\mu_X=\mu$.
\end{theorem}

Appendix~\ref{app:windows} gives the proof. We call $\eoutn{n}$ the \emph{nuisance radius}. It does not require the two windows to behave alike: each enters only through its own Lipschitz constant and slack.

In our example, each report is scored as soon as it is written, so the window is \emph{directly observable} (Definition~\ref{def:direct-observability}), which gives $L_{C_n}=L_{s_n}=1$ and no channel or protocol mismatch. The linear modules give $\xi_{M_A}=\xi_{M_B}=0$, $L_{M_A}=0.8$ and $L_{M_B}=0.5$, and relative to the common bundle $e=0.5$ the entry mismatches are $\epsilon_{\mathrm{in},A}=0.3$ and $\epsilon_{\mathrm{in},B}=0.1$. The bound $\eout=0.8(0.3)+0.5(0.1)=0.29$ therefore covers the discrepancy $|-0.04-0.15|=0.19$.

\subsection{Certification and rank intervals}

\paragraph{Admissibility.}
We first decide whether the evaluation supports the intended comparison, using the declared conditions and the nuisance radius before inspecting the score gap.

\begin{definition}[Readout-relative fair comparability]
\label{def:fair-comparability}
Fix a tolerance $\efairn{n}>0$ before observing any ranking. Windows $M_A$ and $M_B$ are
\emph{$\efairn{n}$-fairly comparable under readout $n$}
if Assumption~\ref{ass:main} holds, the controlled estimand in
\eqref{eq:controlled-contrast} is declared, and
$\eoutn{n}\le\efairn{n}$.
\end{definition}

We call such a pair \emph{admissible}; otherwise the comparison is \emph{refused} under the current design. Fair comparability concerns only the conditions surrounding a window. The stronger claims that two windows behave alike (\emph{equivalence}) or can replace one another with a bounded change in the final outcome (\emph{substitution}) are defined in Appendices~\ref{app:geometry} and~\ref{app:composition}.

\paragraph{Certification.}
For an admissible pair, let $\widehat\Delta_n^{\mathrm{obs}}$ be the estimated contrast and $q_n(\delta')$ its sampling radius at confidence $1-\delta'$. We certify the ordering given by the sign of $\widehat\Delta_n^{\mathrm{obs}}$ only when
\begin{equation}
\big|\widehat\Delta_n^{\mathrm{obs}}\big|
>
q_n(\delta')+\eoutn{n}.
\label{eq:decision}
\end{equation}
If this condition fails, the pair remains \emph{undecided}.

In practice, we replace $\eoutn{n}$ in both decisions with its upper confidence bound, $\heoutn{n}+r_n(\delta)$ (Definition~\ref{def:empirical-fairness}). When entries, channels and protocols are exactly matched, there is no outside mismatch or slack, so the nuisance radius is zero and certification rests on the sampling radius alone.

In our example, the gap of $0.04$ lies within the nuisance radius of $0.29$, so more tasks alone cannot certify the ordering. Once both modules receive the same bundle, with the channel and protocol unchanged, $\eout=0$ and $\Delta^{\mathrm{obs}}=\Delta^\star=0.15$.

Certified relations need not form a total order, so we report each system's \emph{rank interval}, the set of positions it can take in any total order consistent with them (Appendix~\ref{app:estimation}). A rank is \emph{pinned} when its interval contains a single position.

\section{Human Reviews as Structured Ground Truth}
\label{sec:benchmark}

BioLitBench provides structured human references for evaluating generated reviews claim by claim.

\paragraph{Construction and quality.}
We parse $2{,}063$ open-access PubMed Central articles from JATS XML, mostly reviews. Qwen3.5-9B~\citep{DBLP:journals/corr/abs-2505-09388} extracts sentence-level claims with their citations and section assignments. It labels embedding-selected claim pairs with eight relation types: parallel, contrast, elaboration, evidence, mechanism, context, synthesis, and condition. Topically similar pairs from different top-level sections labeled unrelated become hard negatives. We also parse section hierarchies and resolve cited papers to PubMed identifiers and abstracts.

This process yields $2{,}042$ graphs with section, claim, and paper nodes (Table~\ref{tab:biolitbench} in Appendix~\ref{app:validation}). The mean per-article reference resolution rate is $87.4\%$. A PubMedBERT-based NLI model~\citep{DBLP:journals/corr/abs-2007-15779}, fine-tuned on MNLI~\citep{DBLP:conf/naacl/WilliamsNB18} and MedNLI, classifies $87.6\%$ of claims as entailed by at least one cited abstract.

\paragraph{References and splits.}
Each evaluation task has $4$--$9$ same-topic peer reviews to capture variation in acceptable content and organization. The target review supplies the upper normalization anchor; its peers define the reference band used to calibrate tolerances (Appendix~\ref{app:reference}). We split the articles into $1{,}842$ training and $200$ held-out examples. We select $50$ evaluation tasks by clustering held-out topic embeddings and sampling proportionally across clusters. Each task restricts evidence to papers published before its cutoff.

\paragraph{Scoring.}
System outputs and human peers pass through the same observation channel. We match sentences to reference claims using embedding similarity. The threshold is the larger of $0.70$ and the $95$th percentile of off-topic similarities. Citations are resolved in each output's format, and headings match reference sections through fuzzy titles and shared claims. A relation is reproduced when both claims match and their sentence similarity exceeds the $10$th percentile for linked reference claims.

We screen $52$ candidate readouts before pairwise comparisons. The screen checks discrimination, variation, saturation, channel artifacts, and whether a readout measures quality rather than compliance. The $34$ admitted readouts form six axes: retrieval ($6$), synthesis ($6$), reasoning ($13$), planning ($2$), writing ($2$), and form ($5$). Quality scores are normalized between the peers' $10$th percentile and the target review's score. Style scores measure distance from the peer band. The composite weights axes equally.

\paragraph{Entry conditions.}
Under \emph{fixed input}, every system receives the task's reference list and is evaluated on five post-retrieval axes. Under \emph{same pool}, every system searches a frozen corpus of $26.6$M papers, adding the retrieval axis. We report these conditions separately.

\section{Stage-Level Training with Certified Rewards}
\label{sec:scribe}

\emph{SCRIBE} (\textbf{S}tage-\textbf{C}ertified \textbf{R}eview-writing agent with \textbf{I}nterface-\textbf{B}ounded \textbf{E}vidence) combines controlled stage-level rewards with KV-Skill training on a frozen open model.

\paragraph{Controlled stages and rewards.}
After retrieval, synthesis builds a claim graph, planning assigns claims to an outline, and writing drafts the sections. Each stage has a sealed, hashed input and a directly scored output. Receipts link consecutive artifacts. Two versions of a stage entered from the same sealed input and scored at its exit with the same readout and protocol are exactly matched, so $\eout=0$ and $\Delta^{\mathrm{obs}}=\Delta^\star$: the exit score can serve as that stage's reward.

Rewards use BioLitBench readouts at each stage's exit. Training examples and reference graphs come only from the training split; evaluation tasks, their references, and their peers are excluded from those examples. However, reward weights are calibrated on leaderboard tasks using the initial agent's gap to the strongest published pipeline for each readout. Planning is trained; synthesis and writing are scored but keep their initial carriers.

\paragraph{Training and variants.}
KV-Skill~\citep{kvskill2026} provides external operators accessed through a lightweight interface, without adding prompt tokens or changing backbone weights. SCRIBE uses Qwen3.8-27B~\citep{DBLP:journals/corr/abs-2505-09388}. For planning, we register a KV-Skill carrier and optimize it with GRPO on $192$ training tasks. We also report \emph{SCRIBE (untrained)}, the harness with its initial carriers.

\emph{SCRIBE-Luna} runs the same harness on gpt-5.6-luna without stage-level training. This backbone also serves six of the seven published pipelines; DR-Tulu uses its own model. When retrieval is required, SCRIBE uses the shared pool service. Its search budget is capped at the largest baseline budget for that task (Appendix~\ref{app:validation}).
\section{Results}
\label{sec:results}

We evaluate on the $50$ held-out tasks of BioLitBench, which fall into $23$ of the $48$ subfields;
the subfields serve as clusters. The primary observation point truncates each report to the task's
target length (\emph{matched length}); full length is secondary. All analysis constants are recorded
with input hashes.

\subsection{The certified leaderboard}
\label{sec:leaderboard}

\paragraph{Protocol.}
The leaderboard compares complete systems; the backbone is part of each system and is reported on
every row. Within each entry condition all systems share one entry: the reference list under fixed
input, or the frozen corpus under same pool. We call a published pipeline's fixed-input run
\emph{reference-fed} and its same-pool run \emph{self-retrieving}. The nuisance radius is therefore
zero, so certification here is the sampling rule; the bound's role is to show that every pair may be
compared, and Section~\ref{sec:framework-results} shows it refusing pairs that may not. The sampling
radius is a cluster-robust $t$ half-width over the $23$ clusters, Bonferroni-corrected over all pairs
of the table. A pair needs at least $10$ common tasks, and a task either system failed is dropped from
that pair. The two conditions have different entry laws and are never compared with each other.
Figure~\ref{fig:leaderboard} shows the certified rank intervals; Table~\ref{tab:leaderboard} gives the
fixed-input scores (same pool: Table~\ref{tab:leaderboard-pool}).

\begin{figure}[!htb]
\centering
\includegraphics[width=\textwidth]{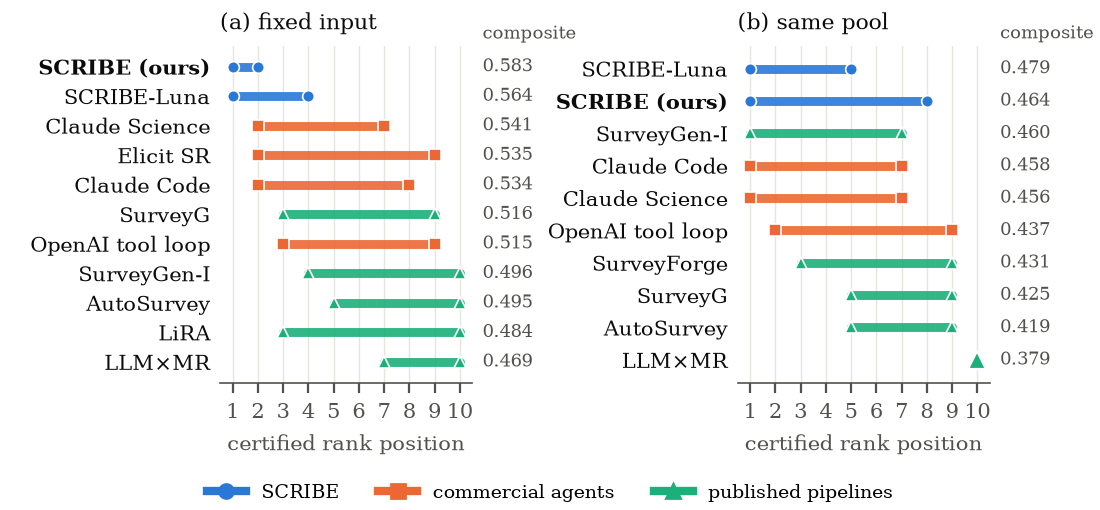}
\caption{\textbf{Certified rank intervals.} Each bar spans all positions a system can occupy in an
ordering consistent with the certified pairwise relations; a single marker denotes a pinned
position. Panel (a) is fixed input and panel (b) is same pool; positions are ranked within a panel
and are not comparable across panels. Systems are sorted by composite score, shown on the right. Most
systems admit an interval of possible ranks rather than a unique rank.}
\label{fig:leaderboard}
\end{figure}

\paragraph{SCRIBE.}
Under fixed input, SCRIBE has a certified rank interval of $[1,2]$ and the highest composite
($0.606$). It is certified above every other reported system, including SCRIBE-Luna, the OpenAI tool
loop, Claude Code, Claude Science and Elicit, except its own untrained harness ($0.583$), from which it
is not separated. The gain over the untrained harness is largest on planning, the trained stage
($0.319$ against $0.276$).

Under same pool, SCRIBE has the highest composite ($0.487$) and a certified interval of $[1,7]$. It is
certified above the OpenAI tool loop, AutoSurvey, SurveyG and LLM$\times$MapReduce (LLM$\times$MR), and
is not separated from SCRIBE-Luna, SurveyGen-I or the Claude agents. Its retrieval score of $0.335$
matches SurveyGen-I's $0.336$.

\begin{table}[!htb]
\centering\footnotesize
\setlength{\tabcolsep}{3pt}
\caption{\textbf{Capability scores} under fixed input at the primary observation point. Rank denotes
the certified rank interval shown in Figure~\ref{fig:leaderboard}. Within each column, the highest
value is bold and the second highest is underlined. Our systems are shaded.}
\label{tab:leaderboard}

% BEGIN tab:leaderboard (a)
% generated by ccbench/quick/make_results_figs.py -- do not edit by hand
\setlength{\tabcolsep}{2.0pt}
\begin{tabular}{@{}p{58pt}p{76pt}Y{41pt}Y{24pt}Y{33.2pt}Y{33.2pt}Y{33.2pt}Y{33.2pt}Y{33.2pt}@{}}
\toprule
& & & & \multicolumn{5}{c}{Capability (fraction of human level)}\\
\cmidrule(l){5-9}
Backbone & System & Composite & Rank & \rot{Synthesis} & \rot{Reasoning} & \rot{Planning} & \rot{Writing} & \rot{Form} \\
\midrule
\multicolumn{9}{c}{\emph{Fixed input: every system receives the review's reference list}}\\
\midrule
\rowcolor{oursrow} \cellcolor{white}Qwen3.8-27B & SCRIBE & \textbf{0.606} & $1$--$2$ & 0.689 & 0.594 & \textbf{0.319} & \textbf{0.613} & \textbf{0.814} \\
\rowcolor{oursrow}  & SCRIBE (untrained) & \underline{0.583} & $1$--$3$ & 0.686 & 0.601 & \underline{0.276} & 0.565 & \underline{0.786} \\
\arrayrulecolor{black!25}\midrule\arrayrulecolor{black}
\rowcolor{oursrow} \cellcolor{white}gpt-5.6-luna & SCRIBE-Luna & 0.564 & $2$--$6$ & \underline{0.700} & \textbf{0.611} & 0.177 & 0.568 & 0.767 \\
 & SurveyG & 0.516 & $4$--$12$ & 0.568 & 0.541 & 0.272 & 0.509 & 0.689 \\
 & OpenAI tool loop & 0.515 & $4$--$12$ & \textbf{0.741} & 0.552 & 0.136 & 0.487 & 0.660 \\
 & SurveyGen-I & 0.496 & $6$--$13$ & 0.591 & 0.510 & 0.182 & 0.449 & 0.746 \\
 & AutoSurvey & 0.495 & $6$--$13$ & 0.638 & 0.533 & 0.135 & 0.483 & 0.687 \\
 & LiRA & 0.484 & $4$--$13$ & 0.539 & 0.495 & 0.254 & 0.442 & 0.716 \\
 & LLM$\times$MR & 0.469 & $10$--$13$ & 0.537 & 0.515 & 0.182 & 0.446 & 0.666 \\
\arrayrulecolor{black!25}\midrule\arrayrulecolor{black}
claude-sonnet-5 & Claude Science & 0.541 & $3$--$10$ & \underline{0.700} & \underline{0.603} & 0.160 & 0.507 & 0.736 \\
 & Claude Code & 0.534 & $3$--$12$ & 0.650 & 0.547 & 0.185 & \underline{0.587} & 0.703 \\
\arrayrulecolor{black!25}\midrule\arrayrulecolor{black}
proprietary & Elicit SR$^{\S}$ & 0.535 & $3$--$10$ & 0.668 & 0.526 & 0.185 & 0.570 & 0.723 \\
\bottomrule
\end{tabular}
% END tab:leaderboard (a)
\vspace{3pt}
\parbox{\textwidth}{\scriptsize Published pipelines use their reference-fed configurations.
$^{\S}$Elicit Systematic Review; it reads full text and cites at most $80$ papers.
Ranks are computed over all systems evaluated, including one commercial system that is not
reported. $n=50$ tasks in $23$ subfield clusters.}
\end{table}

\paragraph{Other systems.}
Under fixed input, Claude Science and Elicit are certified above SurveyGen-I, AutoSurvey and
LLM$\times$MR, and Claude Code above LLM$\times$MR. The two share a backbone, so their difference comes
from the product harness. The published pipelines are separated from the commercial agents, but few from one another.

\paragraph{What the intervals show.}
Most positions in both tables remain intervals, and none is pinned; a conventional leaderboard would
print a distinct rank for every system. The certificate is stable: on a null with no true difference it falsely
certifies near the nominal rate, and leaving out any one subfield reverses no relation
(Appendix~\ref{app:tc-calibration}). A certified relation measures agreement with the human review's
content, not use of the supplied evidence (Appendix~\ref{app:readout-validity}). Under fixed input,
SCRIBE stays first by point estimate at full length (Appendix~\ref{app:leaderboard}). Under same pool,
SurveyGen-I retrieved the target review's own references on $33$ of $50$ tasks; its margins are
unchanged on the other $17$ (Appendix~\ref{app:validation}).

\subsection{Certification versus a conventional analysis}
\label{sec:framework-results}

We apply a conventional analysis and our procedure to the seven published pipelines, with the
backbone held fixed. Table~\ref{tab:rules} adds one requirement at a time.

\begin{table}[!htb]
\centering\small
\caption{\textbf{What each analysis concludes} about the $21$ pairs among the seven published
pipelines at the primary observation point. R4 and R5 are counted over the $10$ admissible pairs.}
\label{tab:rules}
\begin{tabular}{@{}lcl@{}}
\toprule
Rule & Pairs decided & What it adds \\
\midrule
R0\quad sorted table & $21$ / $21$ & any difference counts \\
R1\quad paired $t$-test & $17$ / $21$ & sampling uncertainty \\
R2\quad cluster-robust $t$ & $16$ / $21$ & task clustering \\
R3\quad $+$ Bonferroni & $\mathbf{14}$ / $21$ & multiplicity; a careful conventional analysis \\
\rowcolor{hmB} R4\quad $+$ admissibility test & $7$ / $10$ & $11$ pairs refused before any score is read \\
\rowcolor{hmB} R5\quad $+$ nuisance radius & $\mathbf{7}$ / $10$ & certified; the radius is zero, entries are shared \\
\bottomrule
\end{tabular}
\end{table}

\paragraph{The conventional analysis.}
R3 accounts for clustering and multiplicity, and its false-positive rate is near nominal. Yet it ranks
LiRA first, and LiRA alone is given the target review's bibliography. The system it ranks last,
DR-Tulu, runs its own model. Significance alone does not show that a comparison isolates the
component of interest.

\paragraph{The admissibility test.}
Our test refuses $11$ of the $21$ pairs before any score is read: the $6$ pairs with DR-Tulu, whose
backbone differs, and the $5$ that pair LiRA's entry with a self-retrieving one. These refusals follow
from Eq.~\ref{eq:outside-bound}. For the LiRA pairs, the measured entry mismatch ($0.94$--$0.98$),
$L_M=0.890$ and the score constants $L_{s_n}$ (median $0.83$; $L_C=1$) give a radius of $0.70$--$0.73$
per readout, against a pre-registered tolerance with median $0.21$ (Appendix~\ref{app:tc-refusals}). A backbone change lies outside a
scaffold window, so the DR-Tulu pairs have no finite radius. Seven of the $14$ conventional orderings
fall among them.

\paragraph{What the admissible pairs show.}
With entries matched, the pipelines separate. Among the five self-retrieving pipelines, $7$ of $10$
pairs certify, with SurveyGen-I pinned first and LLM$\times$MR last. Retrieval separates them most: it
certifies $11$ of $21$ same-pool pairs in the capability analysis, and writing none
(Appendix~\ref{app:campaign}). Among the reference-fed configurations only $1$ of $10$ pairs
certifies. Every undecided admissible pair is limited by sampling, not nuisance.

\paragraph{One declared convention changes the winner.}
At matched length, AutoSurvey is third among the reference-fed configurations and nothing is pinned
first. At full length, $6$ of $10$ pairs certify and AutoSurvey is pinned first
(Table~\ref{tab:certified}); its reports have a median of about $250{,}000$ words, against
$10{,}000$--$47{,}000$ for the others. Both views are valid, but they answer different questions, so
the convention must be declared before the comparison is read.

\subsection{The theory on real systems}
\label{sec:theory-results}

Table~\ref{tab:theory-checks} in Appendix~\ref{app:theory-checks} summarizes the tests of the
theory. The attribution bound is tested on the published systems in two designs and holds in both, and
it decides the five LiRA refusals above. It is informative on the writing axis, where no cell is
vacuous, and looser on the others; non-expansiveness depends on the declared length convention
(Appendix~\ref{app:tc-w4}). Published pipelines cannot be re-entered at an intermediate stage, so the
multi-stage composition bound is checked on synthetic chains.

\section{Conclusion}
\label{sec:conclusion}

Compositional controllability decides, before any score is read, whether two agents may be compared at
a chosen window, and certifies an ordering only when the gap exceeds the radius. On BioLitBench, a
careful conventional analysis orders $14$ of the $21$ published-pipeline pairs; our test refuses $11$
pairs, and $7$ of the $14$ orderings fall among them. The same windows give stage-level rewards:
SCRIBE, trained with KV-Skill on Qwen3.8-27B, reaches a certified rank interval of $[1,2]$ under fixed
input, above every published pipeline and every commercial agent we report.

The results have limits. A certified relation measures agreement with the human review's content,
not use of the supplied evidence, and compares systems up to an exit, not modules in isolation. The
length convention can change the winner, so it must be fixed in advance.

\subsection*{AI use statement}

AI assistants were used to support manuscript organization and language revision, and to implement
and debug code. Language models that are part of the research itself, for building BioLitBench and as
the evaluated agents, are described in Sections~\ref{sec:benchmark}--\ref{sec:results}. The authors
reviewed and modified all generated suggestions, verified the experimental results and citations, and
take full responsibility for the final manuscript and released artifacts.

\subsection*{Ethics statement}

This work uses published biomedical literature. It involves no human subjects, patient data or
private information. BioLitBench is built from open-access articles in PubMed Central, whose licences
differ: $1{,}119$ of the $2{,}063$ source articles are CC BY or CC0 and can be redistributed, while
the rest, including NIH author manuscripts, allow text mining but not redistribution. We will
release the redistributable subset together with scripts that rebuild the full benchmark from PubMed
Central. The benchmark measures agreement with a human review's content, not the correctness or
clinical safety of a report (Appendix~\ref{app:readout-validity}). A high certified rank is therefore
not evidence that a generated review is fit for clinical or policy use, and generated reviews should
be checked by experts before any such use. Commercial systems were accessed through their public
interfaces under their terms of service and were evaluated under the same declared conditions as
every other system. We report every refusal and undecided pair rather than a total order, so that
the leaderboard is not read as stronger evidence than it provides.

\subsection*{Reproducibility statement}

Code, the frozen leaderboard inputs and the analysis scripts are available at the
repository linked in the abstract. Section~\ref{sec:method} and Appendix~\ref{app:theory} state every
assumption and give complete proofs. Section~\ref{sec:benchmark} describes how BioLitBench is built
and scored, including the readout admission screen. Section~\ref{sec:scribe} and
Appendix~\ref{app:validation} describe SCRIBE, its backbone, carrier size, training and retrieval
harness. Every published pipeline ran behind a recording proxy that logs each model call and a network
guard that allows only the local corpus service; the per-task logs, integrity hashes of each
pipeline's code and the cost ledgers are part of the release. The decision constants are recorded
with input hashes (Appendix~\ref{app:validation}), so a reader can rerun the analysis and see what
changing them does. The leaderboard tables and Figure~\ref{fig:leaderboard} are generated from the
frozen board by script, and an automated check verifies every table entry against it.
Appendix~\ref{app:theory-checks} gives the setting of each theory check.

\bibliography{iclr2027_conference}
\bibliographystyle{iclr2027_conference}

\appendix

\section{Theory}
\label{app:theory}

This appendix collects the formal setting behind Section~\ref{sec:method} and the proofs of its
results, ordered by what the main text needs. Section~\ref{app:setting} fixes the measure-theoretic
conventions and the model of an agent system. Section~\ref{app:windows} states
Assumption~\ref{ass:main} in full and proves Theorem~\ref{thm:observable-attribution}.
Section~\ref{app:estimation} turns that bound into a finite-sample certificate. The remaining
sections go beyond what the experiments exercise: Section~\ref{app:composition} carries error
across several stages, Section~\ref{app:geometry} builds the behavioural distance in which
equivalence and substitution are stated, and Section~\ref{app:reference} formalises the human
review graph the benchmark scores against. Appendix~\ref{app:validation} validates the campaign.
Appendix~\ref{app:leaderboard} gives the same-pool scores and the full-length leaderboard. Appendix~\ref{app:readout-validity} tests the readouts by evidence
ablation, Appendix~\ref{app:campaign} reports the capability campaign, and
Appendix~\ref{app:theory-checks} the theory checks, and Appendix~\ref{app:scope} records where the
framework stops.

%------------------------------------------------------------------------------------------------
\subsection{Formal setting}
\label{app:setting}

\paragraph{Measurable spaces and kernels.}
All state, action, observation and interface spaces are standard Borel spaces, which guarantees
regular conditional distributions and rules out measure-theoretic pathologies. A Markov kernel
$K:X\rightsquigarrow Y$ assigns to each $x\in X$ a probability measure $K(\cdot\mid x)$ on $Y$,
measurably in $x$. A probability measure $\mu$ on $X$ is pushed through $K$ by
\begin{equation}
(\mu K)(A)=\int_X K(A\mid x)\,\mu(dx),\qquad A\subseteq Y ,
\end{equation}
and for a measurable map $f:X\to Y$ we write $f_{\push}\mu$ for the ordinary pushforward.

\paragraph{Couplings and optimal transport.}
For probability measures $\mu$ on $X$ and $\nu$ on $Y$, let $\Pi(\mu,\nu)$ be the set of couplings
with those marginals. Given a bounded non-negative cost $d:X\times Y\to\R_{\ge0}$,
\begin{equation}
\W_d(\mu,\nu)=\inf_{\Gamma\in\Pi(\mu,\nu)}\int_{X\times Y}d(x,y)\,\Gamma(dx,dy).
\label{eq:wasserstein}
\end{equation}
When $X=Y$ and $d$ is a metric this is the $1$-Wasserstein distance. We keep the same notation when
$X$ and $Y$ differ but share a meaningful ground cost, which is what lets two systems with
different internal spaces be compared at all. Several proofs below select optimal couplings
measurably; Assumption~\ref{ass:ot-regularity} (Appendix~\ref{app:geometry}) is assumed throughout
this appendix.

\paragraph{Amplification factors.}
A Markov kernel $Q:X\rightsquigarrow Y$ has \emph{amplification factor} at most $L$ on a family
$\mathfrak P$ of input laws if $\W_{d_Y}(\nu Q,\nu'Q)\le L\,\W_{d_X}(\nu,\nu')$ for all
$\nu,\nu'\in\mathfrak P$. Then $L$ is a Lipschitz constant of $\nu\mapsto\nu Q$ on $\mathfrak P$,
and this appendix uses the two names interchangeably. The family $\mathfrak P$ is always the laws the benchmark actually
reaches, never all probability measures. Where a kernel has discontinuities, such as a truncation
or a top-$k$ cut-off, we use the affine form
$\W_{d_Y}(\nu Q,\nu'Q)\le L\,\W_{d_X}(\nu,\nu')+\xi$ with a slack $\xi\ge0$. Factors multiply along
a chain: if $Q_1$ and $Q_2$ have factors $L_1$ and $L_2$ on the laws each receives, then $Q_1Q_2$
has factor at most $L_1L_2$.

\paragraph{Pseudometrics.}
A pseudometric is non-negative, symmetric and satisfies the triangle inequality, but may assign
distance zero to distinct objects. Here that is a feature: two histories can be operationally
different and still indistinguishable at the interface an evaluation declares.

\paragraph{Stochastic agent systems.}
The main text treats an agent system as a distribution over executions. Formally:

\begin{definition}[Stochastic agent system]
A finite-horizon stochastic agent system is a tuple
$\cS=(\cC,\cH,\cA,\cO,\iota,\pi,K,H)$, where
\begin{itemize}[leftmargin=1.4em,itemsep=1pt]
\item $\cC$ is a space of task contexts;
\item $\cH=\bigsqcup_{t=0}^{H}\cH_t$ is a disjoint union of history spaces, $\cH_t$ holding
histories of length $t$;
\item $\cA$ and $\cO$ are action and observation spaces;
\item $\iota:\cC\rightsquigarrow\cH_0$ is an initial-history kernel;
\item $\pi:\cH\rightsquigarrow\cA$ is a policy kernel supported on the admissible actions at each
history;
\item $K:\cH\times\cA\rightsquigarrow\cO\times\cH$ is an environment-and-tool kernel whose
next-history component lies in $\cH_{t+1}$ when the current history lies in $\cH_t$;
\item $H<\infty$ is the maximum horizon.
\end{itemize}
Histories that terminate before $H$ are made absorbing.
\end{definition}

The context holds every externally controlled variable relevant to a comparison. For a
literature-review task it may be $c=(q,\cD,\cR,t_{\max},B,\cU)$: the research question, the
accessible corpus, scope and writing requirements, a publication cutoff, a resource budget and the
available tools. The formulation is deliberately more general than a Markov decision process. If a
sufficient state summary exists it may replace the history; if none does, nothing here assumes one.

\paragraph{Branch kernels and computation trees.}
Two systems can carry out the same step through very different internal actions, so their
executions are compared through a shared vocabulary of labels rather than through those actions.
Let $(\cL,d_{\cL})$ be a common label space, whose labels record externally meaningful facts about an
executed branch --- stage, action class, tool class, resource use and observable effect --- and let
each system carry a measurable labelling map $\ell_{\cS}:\cH\times\cA\times\cO\times\cH\to\cL$.
The induced \emph{branch kernel} $P_{\cS}:\cH\rightsquigarrow\cL\times\cH$ is
\begin{equation}
P_{\cS}(D\mid h)=\int_{\cA}\int_{\cO\times\cH}
\ind\big\{(\ell_{\cS}(h,a,o,h'),h')\in D\big\}\,K(do,dh'\mid h,a)\,\pi(da\mid h)
\label{eq:branch-kernel}
\end{equation}
for measurable $D\subseteq\cL\times\cH$. Sampling from $P_{\cS}$ recursively unfolds a rooted
stochastic computation tree whose branches are all possible executions; a rollout is one sampled
root-to-leaf path. Parallel or dependency-structured executions can be written as DAGs and unfolded
into histories without changing the probability distribution over executions.

\paragraph{Semantic interfaces.}
Systems need not share states or actions, but they must expose comparable artifacts.

\begin{definition}[Semantic interface]
Let $(\cZ,d_{\cZ})$ be a common metric space. A \emph{semantic interface} for $\cS$ is a measurable
map $\phi_{\cS}:\cH\to\cZ$. It is \emph{comparison-sufficient} for a family $\cF$ of evaluations if
every $f\in\cF$ factors through it: $f(h)=\bar f(\phi_{\cS}(h))$ for some $\bar f:\cZ\to\R$.
\end{definition}

For a review agent an interface state is $z=(P,E,G,O,b,r)$: a paper set, an evidence ledger, a
synthesis graph, an outline, the remaining budget and a vector of stage-level verification results.
Private chain-of-thought need not be observed.

%------------------------------------------------------------------------------------------------
\subsection{Windows, readouts and the attribution bound}
\label{app:windows}

Section~\ref{sec:method} introduces windows and readouts in enough detail to carry its argument.
This section gives their precise forms, states Assumption~\ref{ass:main} in full, and proves
Theorem~\ref{thm:observable-attribution}.

\paragraph{Frontiers and windows.}
A stage is delimited by when an execution stops, not by where one source file ends.

\begin{definition}[Execution frontier]
An \emph{execution frontier} $F\subseteq\cH$ is a measurable antichain of histories: no member is a
strict prefix of another. An execution reaches $F$ at the stopping time $\tau_F=\inf\{t:h_t\in F\}$.
\end{definition}

\begin{definition}[Comparison window]
A \emph{comparison window} $M=[F^{\mathrm{in}},F^{\mathrm{out}}]$ is an entry frontier and a later
exit frontier such that every evaluated execution entering $F^{\mathrm{in}}$ reaches
$F^{\mathrm{out}}$, or an explicitly labelled failure state, within a bounded number of steps. It
induces the \emph{macro-kernel}
\begin{equation}
Q_{\cS}^{M}:F^{\mathrm{in}}\rightsquigarrow\cX_M\times F^{\mathrm{out}},
\label{eq:macro-kernel}
\end{equation}
where $\cX_M$ records the labelled local trace, cumulative resources and failure indicators inside
the window.
\end{definition}

The internal trace may be a sequence, a dependency DAG, or its tree unfolding. Let $d_M$ be a
bounded, structure-aware discrepancy on $\cX_M\times F^{\mathrm{out}}$ --- for instance the
recursive distance of Section~\ref{app:geometry}, a tree edit distance, or an optimal-transport
distance on attributed execution DAGs \cite{otap2026}. Entry histories are mapped into a common
boundary space $(\cB_M,d_{\cB_M})$ by maps $\psi_X^{\mathrm{in}}$, and exits into a common space
$(\cZ_M,d_M^{\mathrm{out}})$. Throughout, $\cR_M(\rho)\subseteq\cZ_M$ is the set of exit states
reachable under the benchmark distribution $\rho$.

Assumption~\ref{ass:main} applies $Q_X^M$ to laws on $\cB_M$, meaning the exit-interface marginal
of \eqref{eq:macro-kernel} after the two interface maps. Writing it as a kernel on $\cB_M$ rather
than on entry histories presupposes a \emph{boundary factorization}: the exit law depends on the
entry history $h$ only through $\psi_X^{\mathrm{in}}(h)$, so that
$Q_X^M(\cdot\mid h)=Q_X^M(\cdot\mid\psi_X^{\mathrm{in}}(h))$ for almost every entry history reached
under $\rho$. This is stronger than boundary sufficiency, condition (i) below, which only places the
confounders in the boundary: it also rules out any other route from the entry history to the exit.
It is assumed throughout this subsection.

\paragraph{The identification assumption.}
Assumption~\ref{ass:main} has two halves, which do different work. Conditions (i)--(iv) make the
window the only thing the evaluation credits; the constants in (v) bound, on the score scale,
everything else that differs between the two pipelines.

\begin{assumption}[Identification and nuisance bounds]\label{ass:main}
On the declared benchmark support:
\begin{enumerate}[label=(\roman*),leftmargin=2em,itemsep=1pt]
\item \emph{Boundary sufficiency.} Every upstream variable causally influencing both the selected
module and its evaluated outcome lies in the entry boundary or is fixed by protocol.
\item \emph{No hidden bypass.} Every evaluated causal path from the entry frontier to the exit
artifact passes through the window or appears as an explicit branch of its macro-kernel.
\item \emph{Resource accounting.} Tokens, tool calls, latency, monetary cost and accessible tools
that can affect behaviour are matched, or carried in the branch labels and the boundary metric.
\item \emph{Common exit semantics.} Both systems expose exits in $(\cZ_M,d_M^{\mathrm{out}})$, and
$d_M$ dominates exit discrepancy: $d_M((x,z),(x',z'))\ge d_M^{\mathrm{out}}(z,z')$.
\item \emph{Nuisance bounds.} For $X\in\{A,B\}$ there are constants with
\begin{align*}
\W_{d_{\cB_M}}(\mu_X,\mu)&\le\epsilon_{\mathrm{in},X},\\
\W_{d_{\cY_n}}\big(C_{X,n}(\cdot\mid z),C_n(\cdot\mid z)\big)&\le\epsilon_{\mathrm{cont},X,n}
\quad\text{for all }z\in\cR_M(\rho),\\
\W_{d_M^{\mathrm{out}}}(\nu Q_X^M,\nu' Q_X^M)&\le L_{M_X}\W_{d_{\cB_M}}(\nu,\nu')+\xi_{M_X},\\
\W_{d_{\cY_n}}(\nu C_n,\nu' C_n)&\le L_{C_n}\W_{d_M^{\mathrm{out}}}(\nu,\nu');
\end{align*}
$s_n$ is $L_{s_n}$-Lipschitz, and $\epsilon_{\mathrm{protocol},n}$ bounds on the score scale all
remaining resource, alignment, annotation and protocol imbalance. The last two displays need hold
only for the laws $\nu,\nu'$ that occur on the declared benchmark support, including the declared
common entry law $\mu$ and its image $\mu Q_X^M$, not for arbitrary pairs of measures. The slack $\xi_{M_X}\ge0$ absorbs discontinuities such as truncation or a top-$k$
cut-off, as $\xi_k$ does in Assumption~\ref{ass:prop} (Appendix~\ref{app:composition}); $\xi_{M_X}=0$ is the pure Lipschitz case.
\end{enumerate}
\end{assumption}

Each of (i)--(iv) is necessary, and none asks the two modules to behave alike. If boundary
sufficiency fails, a local difference can be an upstream confound. If a path bypasses the window,
credit for it is misattributed to the window. If resources go unaccounted, an apparent improvement
may be additional computation. If exit semantics differ, there is nothing well-defined to compare.

Condition (v) is of a different kind. Conditions (i)--(iv) are qualitative, and if one fails the
bound identifies nothing. Condition (v) is quantitative: it asks that the constants exist but
places no limit on their size. Theorem~\ref{thm:observable-attribution} holds whatever they are, and their
size decides only whether the resulting radius is small enough to pass the tolerance of
Definition~\ref{def:fair-comparability}. Its constants $L_{M_X}$ bound each module's sensitivity
to a perturbation of its entry; they do \emph{not} bound the behavioural difference between $M_A$
and $M_B$.

\paragraph{Readouts.}
For one observed $y\sim\mu_XQ_X^MC_{X,n}$, $s_n(y)$ is a number. Because the window is stochastic, the observed
value $\theta_{X,n}^{\mathrm{obs}}$ of \eqref{eq:observed-value} is a population expectation, and
must in general be estimated (Section~\ref{app:estimation}). A failure or missing readout is
recorded as a declared symbol $\bot_n$ with a prespecified score, never by discarding the failed
run. A readout index $n$ may encode a stage,
a document section and a criterion, for instance
$n=(\text{retrieval},\text{introduction},\text{citation coverage})$.

\begin{definition}[Direct observability]\label{def:direct-observability}
Fix a canonical measurable readout map $r_n:\cZ_M\to\cY_n$. A window $M_X$ is \emph{directly
observable under $n$} if the readout is available at its exit frontier and
$C_{X,n}(\cdot\mid z)=\delta_{r_n(z)}$ on the reachable exit support. Write $\mathfrak D_n$ for the
class of such windows. Equivalently, the score factors through the common exit interface with no
downstream stage.
\end{definition}

Whether a window is directly observable depends on the window, the interface and the readout map.
It does not depend on which runs happened to produce a readout: a window does not become directly
observable because its failed runs are set aside.

\paragraph{Proof of the attribution bound.}

\begin{proof}[Proof of Theorem~\ref{thm:observable-attribution}]
Write the discrepancy as a difference of one term per system:
\begin{equation*}
\Delta_n^{\mathrm{obs}}-\Delta_n^{\star}
=\big[V_n(\mu_AQ_A^MC_{A,n})-V_n(\mu Q_A^MC_n)\big]
-\big[V_n(\mu_BQ_B^MC_{B,n})-V_n(\mu Q_B^MC_n)\big].
\end{equation*}
By the triangle inequality it is enough to bound each bracket and add.

Fix $X$. Its natural pipeline differs from the controlled one in two places, the entry law and the
channel. Insert the law that changes only the channel, $\mu_XQ_X^MC_n$:
\begin{equation*}
\begin{aligned}
V_n(\mu_XQ_X^MC_{X,n})-V_n(\mu Q_X^MC_n)
&=\underbrace{V_n(\mu_XQ_X^MC_{X,n})-V_n(\mu_XQ_X^MC_n)}_{\text{channel}}\\
&\quad+\underbrace{V_n(\mu_XQ_X^MC_n)-V_n(\mu Q_X^MC_n)}_{\text{entry}} .
\end{aligned}
\end{equation*}

\emph{Step 1: from scores to distributions.} Because $s_n$ is $L_{s_n}$-Lipschitz,
Kantorovich--Rubinstein duality gives $|V_n(\lambda)-V_n(\lambda')|\le
L_{s_n}\W_{d_{\cY_n}}(\lambda,\lambda')$ for any two laws on $\cY_n$. Each term is therefore at most
$L_{s_n}$ times a transport distance.

\emph{Step 2: the channel term.} Both laws push the same exit law $\mu_XQ_X^M$, which is
supported on $\cR_M(\rho)$, through different channels. Coupling the two channels optimally at each
exit $z$ and integrating over $z$ gives
$\W(\mu_XQ_X^MC_{X,n},\mu_XQ_X^MC_n)\le\epsilon_{\mathrm{cont},X,n}$.

\emph{Step 3: the entry term.} Both laws share the window and the channel and differ only in
entry. If $\mu_X=\mu$ the two laws are identical and the term is zero, with no appeal to $L_{M_X}$
or $\xi_{M_X}$: the slack is charged only for a system whose entry differs from $\mu$. Otherwise, apply the Lipschitz bound of $C_n$ and then that of $Q_X^M$ with
$\nu=\mu_X$ and $\nu'=\mu$:
\begin{equation*}
\begin{aligned}
\W(\mu_XQ_X^MC_n,\mu Q_X^MC_n)&\le L_{C_n}\W(\mu_XQ_X^M,\mu Q_X^M)
\le L_{C_n}\big(L_{M_X}\W(\mu_X,\mu)+\xi_{M_X}\big)\\
&\le L_{C_n}\big(L_{M_X}\epsilon_{\mathrm{in},X}+\xi_{M_X}\big).
\end{aligned}
\end{equation*}

\emph{Step 4: assemble.} Each bracket is at most
$L_{s_n}\big(L_{C_n}(L_{M_X}\epsilon_{\mathrm{in},X}+\xi_{M_X})+\epsilon_{\mathrm{cont},X,n}\big)$.
Summing over
$X\in\{A,B\}$ and adding $\epsilon_{\mathrm{protocol},n}$, which covers every imbalance the named
constants do not, gives \eqref{eq:outside-bound}.
\end{proof}

\begin{corollary}[Non-expansive windows]\label{cor:nonexpansive}
If $L_{M_A},L_{M_B}\le1$ the windows do not amplify entry mismatch, and the bracket in
\eqref{eq:outside-bound} is at most
$L_{C_n}(\epsilon_{\mathrm{in},A}+\epsilon_{\mathrm{in},B}+\xi_{M_A}+\xi_{M_B})
+\epsilon_{\mathrm{cont},A,n}+\epsilon_{\mathrm{cont},B,n}$. Its multi-stage analogue is
Corollary~\ref{cor:nonexpansive-stages}. Estimating $L_{M_X}$ requires entering the same window
from two different entry laws. For a system window, which begins at the task input, this is
possible: the benchmark sets the task and the searchable corpus, and Appendix~\ref{app:tc-two-entries}
estimates $L_{M_X}$ for the published pipelines by entering each twice on the same task. For a
stage window of a published pipeline it is not, because a black-box pipeline lets the stage's input
artifact be observed but not set (Appendix~\ref{app:scope}).
\end{corollary}

\paragraph{When terms vanish.}
If $M_A,M_B\in\mathfrak D_n$ and share a canonical readout, both channel terms are zero and the
bound reduces to its entry, slack and protocol terms. If one or both windows are read through a
downstream stage, they remain certifiably comparable provided the channels end in the same
observation space with common semantics and $\eoutn{n}\le\efairn{n}$. Two identical channels
contribute no channel mismatch, however long, though a large $L_{C_n}$ still amplifies whatever
entry mismatch remains. If in addition both entry laws equal $\mu$, Step~3 charges neither entry
nor slack, and the radius reduces to $\epsilon_{\mathrm{protocol},n}$. With the protocol matched
too, $\eoutn{n}=0$: this is the exact regime of Section~\ref{sec:leaderboard}.

%------------------------------------------------------------------------------------------------
\subsection{Estimation and certification}
\label{app:estimation}

Theorem~\ref{thm:observable-attribution} bounds the gap between two population quantities. An
evaluation observes neither: it sees finitely many rollouts on finitely many tasks. This section
says what can be certified from them.

\paragraph{A finite-sample certificate.}
Let $c_1,\dots,c_N\overset{\mathrm{iid}}{\sim}\rho$ be task contexts. A context may carry several
paired interventions and repeated rollouts, but the context is the statistical cluster. Let
$\ell(c)\in[0,1]$ be an ideal context-level loss with population mean
$\bar\ell=\E_{c\sim\rho}[\ell(c)]$, let $\ell_i\in[0,1]$ be its rollout-based estimate on $c_i$, and
set $\widehat{\bar\ell}_N=N^{-1}\sum_i\ell_i$. Let $\eta_{\mathrm{roll}}$, $\eta_{\mathrm{align}}$ and
$\eta_{\mathrm{ann}}$ be deterministic bounds on finite-rollout approximation, learned alignment
error and reference annotation error, whose sum bounds
$\big|N^{-1}\sum_i\ell_i-N^{-1}\sum_i\ell(c_i)\big|$. These bounds are \emph{assumed} rather than
derived. The rollout error in particular is random, so a deterministic bound on it is a modelling
assumption that can hold only with high probability.

\begin{proposition}[Finite-sample upper certificate]\label{prop:finite-sample}
For any $\delta\in(0,1)$, with probability at least $1-\delta$,
\begin{equation}
\bar\ell\le\widehat{\bar\ell}_N+\sqrt{\frac{\log(1/\delta)}{2N}}
+\eta_{\mathrm{roll}}+\eta_{\mathrm{align}}+\eta_{\mathrm{ann}}.
\label{eq:certificate}
\end{equation}
\end{proposition}

\begin{proof}
The ideal losses $\ell(c_i)$ are independent and lie in $[0,1]$, so Hoeffding's inequality gives,
with probability at least $1-\delta$,
$\bar\ell\le N^{-1}\sum_i\ell(c_i)+\sqrt{\log(1/\delta)/(2N)}$. Replacing the ideal empirical mean by
$\widehat{\bar\ell}_N$ costs at most $\eta_{\mathrm{roll}}+\eta_{\mathrm{align}}+\eta_{\mathrm{ann}}$
by assumption, and the triangle inequality adds the two.
\end{proof}

Repeated rollouts of one context do not increase $N$: they tighten $\eta_{\mathrm{roll}}$, not the
sampling term. In practice $\eta_{\mathrm{roll}}$ can be estimated from repeated runs and a
hierarchical bootstrap reported beside \eqref{eq:certificate}. The certificate bounds whatever loss
$\ell$ measures. When $\ell$ is a controllability loss it certifies equivalence or a substitution
component --- but not fairness, because such a loss includes the target difference inside the
window.

\paragraph{Certifying a comparison.}
Fairness is certified from an estimate of the nuisance radius itself.

\begin{definition}[Empirically certified fair comparability]\label{def:empirical-fairness}
Let $\heoutn{n}$ estimate \eqref{eq:outside-bound}, with a one-sided radius $r_n(\delta)$ such that
$\Prb\big[\eoutn{n}\le\heoutn{n}+r_n(\delta)\big]\ge1-\delta$. The pair is \emph{empirically
certified $\efairn{n}$-fair under $n$} at confidence $1-\delta$ if
$\heoutn{n}+r_n(\delta)\le\efairn{n}$.
\end{definition}

The precision of the score contrast is a separate matter. If $\widehat\Delta_n^{\mathrm{obs}}$ has
two-sided sampling radius $q_n(\delta')$, then Theorem~\ref{thm:observable-attribution} gives, with
probability at least $1-\delta'$,
\begin{equation}
\Delta_n^{\star}\in\big[\widehat\Delta_n^{\mathrm{obs}}-q_n(\delta')-\eoutn{n},\;
\widehat\Delta_n^{\mathrm{obs}}+q_n(\delta')+\eoutn{n}\big],
\end{equation}
and a direction is certified when this interval excludes zero, which is \eqref{eq:decision}. In
practice $\eoutn{n}$ is unknown and is replaced by its upper bound $\heoutn{n}+r_n(\delta)$, so a
direction is certified when
$\big|\widehat\Delta_n^{\mathrm{obs}}\big|>q_n(\delta')+\heoutn{n}+r_n(\delta)$. The interval then
contains $\Delta_n^{\star}$ whenever both the sampling event and the event of
Definition~\ref{def:empirical-fairness} hold, which by a union bound has probability at least
$1-\delta-\delta'$. A certified direction is therefore wrong with probability at most
$\delta+\delta'$. In exact cells the radius is zero by construction rather than estimated, so
$r_n=0$ and only $\delta'$ is spent.

A pair that fails this test is \emph{undecided}, for one of two reasons. It is
\emph{attribution-limited} when the population gap $|\Delta_n^{\mathrm{obs}}|$ lies within the
nuisance radius $\eoutn{n}$, so that more samples alone cannot certify the ordering, and
\emph{sampling-limited} when the gap exceeds the nuisance radius but sampling uncertainty prevents
certification.

\paragraph{Rank intervals.}
Certified relations define a partially resolved relation rather than a total order. Because
pair-specific radii differ, the certified relation need not be transitive, and in principle it could
contain a cycle. When it is acyclic, its transitive closure is a strict partial order, and a system's
\emph{rank interval} is the range of positions it takes across all linear extensions of that order.
Provided the sampling radii are valid, the family-wise correction makes every certified relation hold
simultaneously with the stated probability, and the relations the closure adds hold on the same event,
since true utilities are ordered transitively. Appendix~\ref{app:tc-calibration} checks the realised
rate against a sign-flip null.

%------------------------------------------------------------------------------------------------
\subsection{Composition across stages}
\label{app:composition}

Theorem~\ref{thm:observable-attribution} concerns one window. A pipeline has several stages, and a
difference introduced at one is carried --- and possibly amplified --- by every stage after it. This
section bounds that propagation. It is stated in general; testing it needs every stage to be run
from a controlled input, which published pipelines do not allow: their intermediate artifacts can be
recovered from logs but not set (Appendix~\ref{app:scope}).

\paragraph{Setting.}
Let $(\cZ_k,d_k)$, $k=0,\dots,K$, be the interfaces between $K$ stages. System $X\in\{A,B\}$ has
stage kernels $Q_{X,k}:\cZ_{k-1}\rightsquigarrow\cZ_k$ and initial law $\nu_{X,0}$, and its law at
interface $k$ is $\nu_{X,k}=\nu_{X,k-1}Q_{X,k}$. Let $\cR_k(\rho)\subseteq\cZ_k$ be the interface
states reachable under the benchmark and its declared interventions, and $\mathfrak P_k(\rho)$ an
admissible family of laws supported there. Every law below, including those in the proofs, is
assumed admissible; this avoids claiming stability on histories no benchmark reaches.

\begin{assumption}[Local replacement error]\label{ass:local-error}
For each $k$ there is $\epsilon_k\ge0$ with
$\sup_{z\in\cR_{k-1}(\rho)}\W_{d_k}\big(Q_{A,k}(\cdot\mid z),Q_{B,k}(\cdot\mid z)\big)\le\epsilon_k$.
\end{assumption}

\begin{assumption}[Approximate downstream stability]\label{ass:prop}
For each $k$ there are $L_k,\xi_k\ge0$ with
\begin{equation}
\W_{d_k}(\nu Q_{B,k},\nu' Q_{B,k})\le L_k\W_{d_{k-1}}(\nu,\nu')+\xi_k
\qquad\text{for all }\nu,\nu'\in\mathfrak P_{k-1}(\rho).
\label{eq:kernel-lipschitz}
\end{equation}
\end{assumption}

The slack $\xi_k$ absorbs discontinuities a benchmark actually meets --- routing thresholds, top-$k$
changes, parsing failures, context truncation. System $B$'s kernels serve only as a reference path;
a symmetric result may use either system's, or a common upper bound. The constants belong to the
declared interfaces, metrics and support, not to the implementation alone, and no condition
$L_k<1$ is needed for a finite-stage bound.

\begin{theorem}[Compositional substitution bound]\label{thm:composition}
Under Assumptions~\ref{ass:local-error} and~\ref{ass:prop}, if
$\W_{d_0}(\nu_{A,0},\nu_{B,0})\le\epsilon_0$ then for every $m\in\{1,\dots,K\}$
\begin{equation}
\W_{d_m}(\nu_{A,m},\nu_{B,m})\le\epsilon_0\prod_{j=1}^{m}L_j
+\sum_{k=1}^{m}(\epsilon_k+\xi_k)\prod_{j=k+1}^{m}L_j ,
\label{eq:composition-bound}
\end{equation}
an empty product being $1$.
\end{theorem}

\begin{proof}
Write $\Delta_k=\W_{d_k}(\nu_{A,k},\nu_{B,k})$. The idea is to change one thing at a time: first the
stage, then the input to it. Insert $\nu_{A,k-1}Q_{B,k}$, which runs system $B$'s stage on system
$A$'s input:
\begin{equation*}
\Delta_k\le\W_{d_k}(\nu_{A,k-1}Q_{A,k},\nu_{A,k-1}Q_{B,k})
+\W_{d_k}(\nu_{A,k-1}Q_{B,k},\nu_{B,k-1}Q_{B,k}).
\end{equation*}
The first term compares the two stages on a common input. Coupling them optimally at each
$z\in\cR_{k-1}(\rho)$ and integrating over $\nu_{A,k-1}$ bounds it by $\epsilon_k$. The second term
runs one stage on two inputs, and \eqref{eq:kernel-lipschitz} bounds it by $L_k\Delta_{k-1}+\xi_k$.
Hence
\begin{equation}
\Delta_k\le\epsilon_k+\xi_k+L_k\Delta_{k-1},\qquad\Delta_0\le\epsilon_0 .
\label{eq:error-recursion}
\end{equation}
Unrolling from $k=m$ down to $k=1$, each error $\epsilon_k+\xi_k$ is multiplied by the constants of
the stages after it, and $\epsilon_0$ by all of them, which is \eqref{eq:composition-bound}.
\end{proof}

\begin{corollary}[Non-expansive stages]\label{cor:nonexpansive-stages}
If $L_k\le1$ for every $k$, then
$\W_{d_K}(\nu_{A,K},\nu_{B,K})\le\epsilon_0+\sum_{k=1}^{K}(\epsilon_k+\xi_k)$: errors add, and no
stage amplifies them.
\end{corollary}

When two systems differ only on a block of stages, the same recursion gives a bound for the block
propagated to the end of the pipeline.

\begin{corollary}[Propagated bound for a window]\label{cor:window}
Suppose $A$ and $B$ differ on stages $i{:}j$, with
$\W_{d_{i-1}}(\nu_{A,i-1},\nu_{B,i-1})\le\epsilon_{i-1}$, and the common downstream stages
after stage $j$ have constant $L_{j+1:K}=\prod_{k=j+1}^{K}L_k$. Then the terminal discrepancy is at most
\begin{equation}
\epsilon_{\mathrm{subst}}^{\mathrm{global}}
=L_{j+1:K}\Big[\epsilon_{i-1}\prod_{k=i}^{j}L_k
+\sum_{k=i}^{j}(\epsilon_k+\xi_k)\prod_{r=k+1}^{j}L_r\Big]
+\sum_{k=j+1}^{K}\xi_k\prod_{r=k+1}^{K}L_r .
\label{eq:window-bound}
\end{equation}
Any independently bounded difference between the two systems' downstream stages must also be
propagated to the end and added.
\end{corollary}

Because \eqref{eq:window-bound} contains the in-window errors $\epsilon_i,\dots,\epsilon_j$, it is a
substitution bound, not a fairness criterion. The two windows are \emph{$\epsilon$-substitutable}
in the pipeline if \eqref{eq:window-bound} is at most $\epsilon$: replacing one with the other then
moves the pipeline's terminal law by at most $\epsilon$ in $d_K$. The same bound controls any downstream quantity that
is Lipschitz in the final interface.

\begin{theorem}[Utility stability]\label{thm:utility}
Let $U:\cZ_K\to\R$ be $L_U$-Lipschitz with respect to $d_K$. Under the conditions of
Theorem~\ref{thm:composition},
\begin{equation}
\big|\E_{\nu_{A,K}}[U]-\E_{\nu_{B,K}}[U]\big|
\le L_U\Big[\epsilon_0\prod_{j=1}^{K}L_j+\sum_{k=1}^{K}(\epsilon_k+\xi_k)\prod_{j=k+1}^{K}L_j\Big].
\end{equation}
\end{theorem}

\begin{proof}
Kantorovich--Rubinstein duality bounds the difference of expectations by
$L_U\W_{d_K}(\nu_{A,K},\nu_{B,K})$, and Theorem~\ref{thm:composition} with $m=K$ bounds the
distance.
\end{proof}

%------------------------------------------------------------------------------------------------
\subsection{Behavioural geometry: equivalence and substitution}
\label{app:geometry}

Fairness constrains only what surrounds a window. The two stronger claims the main text defers ---
that two windows behave alike, and that one can replace the other --- constrain what happens inside
it. Stating them needs a distance between the behaviours of systems that share neither states nor
actions. This section constructs one and shows it is well defined.

\paragraph{The controllability operator.}
Consider systems $\cS_A,\cS_B$ with history spaces $\cH_A,\cH_B$, branch kernels $P_A,P_B$ and
interfaces $\phi_A,\phi_B$ into shared spaces $(\cL,d_{\cL})$ and $(\cZ,d_{\cZ})$, with
$d_{\cL},d_{\cZ}\in[0,1]$.

\begin{assumption}[Optimal-transport regularity]\label{ass:ot-regularity}
Every ground cost below is bounded and measurable, every displayed optimal-transport infimum is
attained, and each optimal value is measurable in its conditioning histories. These hold
automatically for finite discretisations, and on Polish spaces under standard lower-semicontinuity
and tightness conditions.
\end{assumption}

Fix a continuation weight $\gamma\in[0,1)$ and a label weight $\beta\in(0,1]$. For a bounded
$d:\cH_A\times\cH_B\to[0,1]$, the \emph{lifted branch cost} compares two next branches by their
labels and their successors:
\begin{equation}
\bar d\big((\ell_A,h'_A),(\ell_B,h'_B)\big)=\beta\,d_{\cL}(\ell_A,\ell_B)+(1-\beta)\,d(h'_A,h'_B).
\label{eq:lifted-cost}
\end{equation}

\begin{definition}[Controllability operator]
The operator $\mathcal T_{AB}$ acts on bounded cross-system distances by
\begin{equation}
(\mathcal T_{AB}d)(h_A,h_B)=(1-\gamma)\,d_{\cZ}\big(\phi_A(h_A),\phi_B(h_B)\big)
+\gamma\,\W_{\bar d}\big(P_A(\cdot\mid h_A),P_B(\cdot\mid h_B)\big).
\label{eq:ctrl-operator}
\end{equation}
\end{definition}

The first term measures how far apart the two systems are now; the second finds the cheapest
probabilistic alignment of what they do next. Two systems can therefore use different tools or
decompose the work differently and still be close, provided their labels, effects and future
semantic states can be coupled cheaply.

\begin{theorem}[Existence and uniqueness]\label{thm:fixed-point}
Under Assumption~\ref{ass:ot-regularity}, $\mathcal T_{AB}$ maps bounded measurable functions
$\cH_A\times\cH_B\to[0,1]$ into themselves and is a contraction in the supremum norm with modulus at
most $\gamma(1-\beta)<1$. It therefore has a unique fixed point $d_{AB}^{\dagger}=\mathcal
T_{AB}d_{AB}^{\dagger}$.
\end{theorem}

\begin{proof}
\emph{Self-map.} Since $d_{\cL}$, $d_{\cZ}$ and $d$ take values in $[0,1]$, so does the lifted cost
\eqref{eq:lifted-cost}, hence so does its transport value, and $\mathcal T_{AB}d$ is a convex
combination of two $[0,1]$-valued terms.

\emph{Contraction.} Take candidate distances $d_1,d_2$. Their lifted costs differ pointwise by at
most $(1-\beta)\|d_1-d_2\|_\infty$, since only the successor term changes. For fixed marginals,
perturbing a transport cost by at most $r$ everywhere moves its optimum by at most $r$ --- any
coupling's cost moves by at most $r$, and so does the infimum. The interface term does not depend on
$d$ at all. Hence
$\|\mathcal T_{AB}d_1-\mathcal T_{AB}d_2\|_\infty\le\gamma(1-\beta)\|d_1-d_2\|_\infty$.

\emph{Fixed point.} Bounded measurable $[0,1]$-valued functions form a complete metric space under
the supremum norm, and Banach's fixed-point theorem gives a unique fixed point.
\end{proof}

\begin{definition}[Recursive controllability distance]
The fixed point $d_{AB}^{\dagger}$ of Theorem~\ref{thm:fixed-point} is the \emph{recursive
controllability distance} between histories of $A$ and $B$.
\end{definition}

A distance between two systems' histories is more useful if it is a single pseudometric over all
systems at once, so that the triangle inequality holds across a whole panel.

\begin{proposition}[Pseudometric realisation]\label{prop:pseudometric}
Let $\cH_{\sqcup}=\bigsqcup_{X\in\cI}\cH_X$ be the disjoint union of the history spaces of a
collection of systems, with $P(\cdot\mid h)=P_X(\cdot\mid h)$ and $\phi(h)=\phi_X(h)$ for
$h\in\cH_X$. If $d_{\cL}$ and $d_{\cZ}$ are pseudometrics, the fixed point of
\eqref{eq:ctrl-operator} on $\cH_{\sqcup}\times\cH_{\sqcup}$ is a pseudometric, and its restriction
to $\cH_A\times\cH_B$ is $d_{AB}^{\dagger}$.
\end{proposition}

\begin{proof}
Iterate from $d_0\equiv0$, which is a pseudometric, and show the property survives each step.
Suppose $d_n$ is a pseudometric. Then the lifted cost $\bar d_n$ is a pseudometric on
$\cL\times\cH_{\sqcup}$, as a non-negative combination of two. The Wasserstein lift of a
pseudometric is a pseudometric on probability measures, and its pullback through
$h\mapsto P(\cdot\mid h)$ is again one. The interface term is the pullback of $d_{\cZ}$ through
$\phi$. Their non-negative weighted sum $d_{n+1}=\mathcal Td_n$ is therefore a pseudometric. By
Theorem~\ref{thm:fixed-point} the iterates converge uniformly, and non-negativity, symmetry and the
triangle inequality are all preserved under uniform limits. Uniqueness of the fixed point gives the
restriction claim.
\end{proof}

\paragraph{From histories to systems.}
Comparing two whole systems means comparing their initial histories task by task.

\begin{definition}[Whole-system controllability]\label{def:whole-ctrl}
For a task distribution $\rho$ on $\cC$,
\begin{equation}
\Ctrl(\cS_A,\cS_B;\rho)=\E_{c\sim\rho}\Big[\W_{d_{AB}^{\dagger}}\big(\iota_A(\cdot\mid c),
\iota_B(\cdot\mid c)\big)\Big].
\label{eq:whole-ctrl}
\end{equation}
Coupling within each context prevents an easy task for one system being matched with a different
task for the other. The uniform version $\Ctrl^{\infty}$ replaces the expectation over $c$ by the
$\rho$-essential supremum.
\end{definition}

Distance zero means what it should, provided the underlying costs vanish only on genuinely
equivalent objects.

\begin{proposition}[Exact behavioural equivalence]\label{prop:zero-distance}
Let $\gamma\in(0,1)$ and $\beta\in(0,1)$. Suppose $d_{\cZ}(z,z')=0$ only when $z,z'$ are evaluation-equivalent, and $d_{\cL}(\ell,\ell')=0$
only when $\ell,\ell'$ are functionally equivalent. If $d_{AB}^{\dagger}(h_A,h_B)=0$, then the
current interface states are evaluation-equivalent, and there is a coupling of the next branches
supported only on functionally equivalent labels and on successor pairs at recursive distance zero.
\end{proposition}

\begin{proof}
Both terms on the right of \eqref{eq:ctrl-operator} are non-negative and carry positive weight,
since $1-\gamma>0$ and $\gamma>0$, so both vanish. The interface term vanishing gives
evaluation-equivalence. For the transport term, attainment (Assumption~\ref{ass:ot-regularity})
gives an optimal coupling of cost zero; since $\beta>0$ and $1-\beta>0$, that coupling must be
supported where both the label cost and the successor cost are zero.
\end{proof}

\begin{remark}[Relation to bisimulation]
When the systems share an MDP's states and actions, labels identify actions, and $d_{\cZ}$ measures
immediate reward difference, \eqref{eq:ctrl-operator} reduces to a bisimulation-style recursive
metric \cite{ferns2004metrics,ferns2011bisimulation}. Because the branch laws are those of the
systems' own policies rather than a maximum over actions, the closest analogue is the on-policy
($\pi$-)bisimulation metric \cite{castro2020scalable}. The cross-system form replaces the maximum
over shared actions with an optimal coupling over heterogeneous labelled branches.
\end{remark}

\paragraph{Equivalence of partial systems.}
The same construction applies to a window. Let $\eta_A^M,\eta_B^M$ be the entry-history laws
induced by $\rho$ and the upstream parts of $A$ and $B$.

\begin{definition}[Controllability radius and $\epsilon$-controlled equivalence]\label{def:equiv}
For $\lambda_{\mathrm{in}}\in(0,1)$,
\begin{equation}
\begin{aligned}
\Ctrl_M(A,B;\rho)=\inf_{\Lambda\in\Pi(\eta_A^M,\eta_B^M)}\int\Big[
&\lambda_{\mathrm{in}}\,d_{\cB_M}\big(\psi_A^{\mathrm{in}}(h_A),\psi_B^{\mathrm{in}}(h_B)\big)\\
&+(1-\lambda_{\mathrm{in}})\,\W_{d_M}\big(Q_A^M(\cdot\mid h_A),Q_B^M(\cdot\mid h_B)\big)
\Big]\Lambda(dh_A,dh_B).
\end{aligned}
\label{eq:partial-radius}
\end{equation}
The windows are \emph{$\epsilon$-controllably equivalent} under $\rho$ if
$\Ctrl_M(A,B;\rho)\le\epsilon$, and \emph{uniformly} so if the same holds with the essential
supremum over the optimal entry coupling in place of the integral.
\end{definition}

The infimum in \eqref{eq:partial-radius} ranges over all couplings of the entry laws, so it may pair
an entry reached on one task with an entry reached on a different task. When the comparison is meant
to be task-matched, as in Definition~\ref{def:whole-ctrl}, we use the \emph{contextwise} version
$\Ctrl_M^{\rho}$, in which $\Lambda$ is restricted to couplings that are the identity on the task
context $c\sim\rho$ and arbitrary within it. Restricting the feasible set can only raise the value,
so $\Ctrl_M\le\Ctrl_M^{\rho}$.

$\Ctrl_M$ is the smallest residual radius the declared interface, branch cost and task distribution
support; it does not claim the implementations are identical. Because it includes differences
\emph{inside} the window, it supports equivalence and substitution but is not a fairness criterion.

\begin{remark}[Fairness, equivalence and substitution]\label{rem:fairness-vs-equivalence}
Fairness requires only $\eoutn{n}\le\efairn{n}$ (Definition~\ref{def:fair-comparability}).
Approximate equivalence additionally requires a small $\Ctrl_M$, and $\epsilon$-substitutability
(Section~\ref{app:composition}) a small propagated bound \eqref{eq:window-bound}. A large controlled module effect is therefore
compatible with a perfectly fair comparison.
\end{remark}

%------------------------------------------------------------------------------------------------
\subsection{Reference specifications from human review graphs}
\label{app:reference}

\begin{definition}[Executable reference specification]
\label{def:reference-spec}
For each reference boundary state $b\in\cB^\star$, let $\cK^\star(b)$ be a nonempty set of acceptable continuation kernels over the next standardized artifact. Together with the reference graph $G^\star$, these define an \emph{executable reference specification}
$\cS^\star=(\cB^\star,\cK^\star,G^\star)$.
\end{definition}

The benchmark scores every system against a human review. A review is not one correct execution,
and treating it as one would penalise valid alternative syntheses. This section makes the reference
precise.

\paragraph{The human review graph.}
The reference artifact for scientific synthesis is neither a report string nor a trajectory, but a
structured account of what the review claims and on what evidence.

\begin{definition}[Human review graph]
A \emph{human review graph} is a typed attributed directed graph
$G^{\star}=(V^{\star},E^{\star},\tau_V,\tau_E,a_V,w_V)$. The map $\tau_V$ assigns node types in
$\{\text{topic},\text{claim},\text{evidence},\text{paper},\text{condition}\}$; $\tau_E$ assigns
relation types such as \texttt{supports}, \texttt{contradicts}, \texttt{qualifies},
\texttt{evidence-from}, \texttt{belongs-to-topic} and \texttt{valid-under-condition}; $a_V$ holds
semantic and provenance attributes; and $w_V\ge0$ holds importance weights.
\end{definition}

The graph is generally a DAG rather than a tree: one paper may support several claims, one claim may
synthesise several pieces of evidence, and one condition may qualify several claims. A topic
hierarchy or outline can remain a tree inside it.

\paragraph{Graph discrepancy.}
Let $G=(V,E)$ and $G'=(V',E')$ carry node masses $m\in\Delta^{|V|}$ and $m'\in\Delta^{|V'|}$, where
$\Delta^p$ is the probability simplex. Let $C_{ii'}$ compare node type, semantics, evidence identity
and provenance, and let $R^G_{ij}$, $R^{G'}_{i'j'}$ encode typed relational dissimilarity. A fused
discrepancy that matches nodes and relations together is
\begin{equation}
d_{\mathrm{FGW}}^2(G,G')=\inf_{\Gamma\in\Pi(m,m')}\Big\{(1-\lambda_G)\sum_{i,i'}C_{ii'}\Gamma_{ii'}
+\lambda_G\sum_{i,j,i',j'}\big|R^G_{ij}-R^{G'}_{i'j'}\big|^2\Gamma_{ii'}\Gamma_{jj'}\Big\},
\label{eq:fgw}
\end{equation}
with $\lambda_G\in[0,1]$. This is the fused Gromov--Wasserstein objective of
\cite{vayer2019fgw}, with a linear feature term and a squared structure term. The square on the
left is notational: the right-hand side is not itself the square of a metric, and
$d_{\mathrm{FGW}}$ is a metric only under the conditions stated there, satisfying in general a
relaxed triangle inequality. Paper identifiers and evidence spans can act as hard anchors, while claim
and topic nodes are aligned semantically; null nodes or unbalanced transport represent omitted or
novel material. The choice of metric must be validated against expert judgement. \eqref{eq:fgw} is one explicit instantiation, not a claim that it captures all graph semantics.

\paragraph{Set-valued reference behaviour.}
A human review specifies acceptable synthesis only partially, so the reference is a \emph{set} of
acceptable continuations --- the executable reference specification
$\cS^{\star}=(\cB^{\star},\cK^{\star},G^{\star})$ of Definition~\ref{def:reference-spec}. The set
$\cK^{\star}(b)$ can admit alternative searches, paper sets, taxonomies and organisations while
still enforcing evidence support, scope, contradiction handling and provenance.

To measure a system against a set, compare it with the closest acceptable member. Let $\nu^{\star}$
be a reference boundary-state distribution with $\supp(\nu^{\star})=\cB^{\star}$, a non-empty compact
subset of a common boundary space $(\cB,d_{\cB})$, and fix a measurable nearest-reference selector
\begin{equation}
a^{\star}(b)\in\arg\min_{b^{\star}\in\cB^{\star}}d_{\cB}(b,b^{\star}),
\label{eq:reference-selector}
\end{equation}
whose existence is a modelling assumption that holds under standard compactness and measurability
conditions. Kernels at $b$ and at $a^{\star}(b)$ produce artifacts in a common output space
$(\cZ_{\mathrm{out}},d_{\mathrm{out}})$. For a system with continuation kernel $Q(\cdot\mid b)$, its
\emph{reference error} at $b\in\cB$ is
\begin{equation}
e(b)=\inf_{Q^{\star}\in\cK^{\star}(a^{\star}(b))}\W_{d_{\mathrm{out}}}\big(Q(\cdot\mid b),Q^{\star}\big).
\label{eq:reference-error}
\end{equation}
Likewise, $\Ctrl(\cS,\cS^{\star};\rho)$ for a set-valued reference means the infimum of
\eqref{eq:whole-ctrl} over all measurable executable selections $Q^{\star}(\cdot\mid b)\in
\cK^{\star}(b)$. Both compare a system with the nearest acceptable behaviour rather than with one
arbitrarily chosen human trajectory.

\begin{remark}[What the reference does not specify]
The review graph does not normally identify private reasoning or one exact tool-call sequence.
Claims about internal process equality are therefore valid only at declared observable interfaces.
Evaluating unrestricted chain-of-thought would need further assumptions, and the theory does not
require it.
\end{remark}

%================================================================================================
\section{Campaign validation}
\label{app:validation}

Table~\ref{tab:biolitbench} summarizes BioLitBench.

\begin{table}[H]
\centering
\small
\caption{\textbf{BioLitBench at a glance.} Statistics for the $2{,}042$ usable graphs.}
\label{tab:biolitbench}
\begin{tabular}{@{}lr@{\qquad}lr@{}}
\toprule
Articles & $2{,}042$ &
Fields / subfields / topics & $6$ / $48$ / $230$ \\
Claims & $246{,}102$ &
Median claims per article & $108$ \\
Typed relations & $832{,}487$ &
Hard negatives & $103{,}242$ \\
Unique cited papers & $224{,}140$ &
Review publication years & $2016$--$2025$ \\
Training / held-out articles & $1{,}842$ / $200$ &
Evaluation tasks & $50$ \\
\bottomrule
\end{tabular}
\end{table}

\paragraph{SCRIBE.}
SCRIBE's backbone is a frozen $27$B open model (Qwen3.8-27B, bf16, greedy decoding at seed $0$,
$65{,}536$-token context for generation), and its trained object is a KV-Skill carrier of $14$--$16$M parameters per trained stage,
served through a patched inference server. It exposes three windows --- synthesis, planning
and writing --- each emitting its exit artifact. Its entry is sealed before the first model call,
and a check at the start of synthesis fails the run if any retrieval capability is still reachable,
so the absence of search downstream holds by construction. Each window appends a receipt whose exit
hash is the next window's entry hash, so a claim about a run can be checked against the run itself.
The trained SCRIBE runs the current pipeline. With the same pipeline and its initial carriers, the composite is $0.595$ under fixed input and $0.494$ under same pool, so the planning carrier's paired gain is $+0.011$ and $-0.010$, neither certified; SCRIBE (untrained) in the tables is the earlier pipeline.

\paragraph{SCRIBE's retrieval harness.}
Under same pool SCRIBE retrieves in two stages. The model first decomposes the question into
subtopics, proposes queries under a duplicate guard and ranks the candidates. A selection rule then
fills the bundle from the top $1{,}000$ ranked papers in two-year windows anchored at the cutoff,
ordering each window by in-pool citation count with citers strictly before the cutoff and the
evaluated review removed. That rule reads the corpus citation graph directly, so it is part of the
system rather than a tool call, and about $40\,\%$ of the final bundle is chosen by it rather than by
the model's own ranking. The per-task search, row and document budgets are each the largest that any
baseline spent on that task. The harness was selected on screening and held-out tasks disjoint from
the $50$ evaluation tasks. The earlier harness, without the selection rule, is the SCRIBE row of
Table~\ref{tab:capability-pool}; the leaderboard (Section~\ref{sec:leaderboard}) uses the final one.

\paragraph{Entry disclosures.}
Three facts bear on the same-pool table (Table~\ref{tab:leaderboard-pool}), and the fixed-input table is untouched by all of them.
(i) The pool service the published pipelines ran exposes a citation route on which the evaluated
review's own reference list can be requested; SurveyGen-I requested it on $33$ of $50$ tasks,
and its paired margins are the same on the $17$ tasks where it did not (e.g.\ $+0.080$ and
$+0.081$ over LLM$\times$MR), so the route does not produce its position.
(ii) SCRIBE is restricted by a route guard to plain search and title, abstract and year, and the
commercial agents' tools exclude the evaluated review. (iii) A SCRIBE
acquisition that stops at its document budget is scored as a normal exit when its bundle is sealed,
its hashes match, its audits are clean and all three downstream windows succeed; the published
pipelines' own budget stops are likewise recorded as normal.

\paragraph{Report length.}
Report length is a genuine confounder: the primary view holds every report to the task's target
length, and on entry-conditioned synthesis the full-length view reverses the conclusion, certifying
a pipeline whose reports have a median of about $250{,}000$ words over SCRIBE's reports of under
$5{,}000$. All analysis constants are recorded with input hashes, so a reader can rerun the analysis
and see what changing them does.

\paragraph{Assertions.}
Twenty-nine assertions run with every rebuild and all pass; four carry most of the weight. The
table reproduces the previous analysis cell by cell ($288$ cells,
$\max|\Delta|=1.11\times10^{-16}$). A replicate null --- three seeds of one model against each
other, where nothing should separate --- certifies $0$ of $114$ cells, the control showing the
threshold has not been loosened. Entry byte-equality is re-derived rather than assumed and holds on
$50$ of $50$ tasks, without which the fixed-input certifications would not stand. And rescoring the
reference-fed reports on an independent implementation differs on $4$ of $81{,}575$ cells, all on
readouts the admission rule already excludes.

\section{Further leaderboard results}
\label{app:leaderboard}

This appendix gives the same-pool capability scores behind Figure~\ref{fig:leaderboard}(b)
(Table~\ref{tab:leaderboard-pool}) and the leaderboard at full length.

\begin{table}[H]
\centering\footnotesize
\setlength{\tabcolsep}{3pt}
\caption{\textbf{Capability scores under same pool} at the primary observation point. Conventions as in Table~\ref{tab:leaderboard}.}
\label{tab:leaderboard-pool}

% BEGIN tab:leaderboard (b)
% generated by ccbench/quick/make_results_figs.py -- do not edit by hand
\setlength{\tabcolsep}{2.0pt}
\begin{tabular}{@{}p{58pt}p{76pt}Y{41pt}Y{24pt}Y{27.0pt}Y{27.0pt}Y{27.0pt}Y{27.0pt}Y{27.0pt}Y{27.0pt}@{}}
\toprule
& & & & \multicolumn{6}{c}{Capability (fraction of human level)}\\
\cmidrule(l){5-10}
Backbone & System & Composite & Rank & \rot{Retrieval} & \rot{Synthesis} & \rot{Reasoning} & \rot{Planning} & \rot{Writing} & \rot{Form} \\
\midrule
\multicolumn{10}{c}{\emph{Same pool: every system searches one frozen corpus}}\\
\midrule
\rowcolor{oursrow} \cellcolor{white}Qwen3.8-27B & SCRIBE & \textbf{0.487} & $1$--$7$ & \underline{0.335} & 0.523 & 0.545 & 0.250 & 0.446 & 0.826 \\
\rowcolor{oursrow}  & SCRIBE (untrained) & 0.464 & $1$--$9$ & \underline{0.335} & 0.512 & 0.538 & 0.232 & 0.404 & 0.766 \\
\arrayrulecolor{black!25}\midrule\arrayrulecolor{black}
\rowcolor{oursrow} \cellcolor{white}gpt-5.6-luna & SCRIBE-Luna & \underline{0.479} & $1$--$6$ & \underline{0.335} & 0.577 & \textbf{0.550} & 0.185 & 0.386 & \underline{0.842} \\
 & SurveyGen-I & 0.460 & $1$--$9$ & \textbf{0.336} & 0.535 & 0.485 & 0.181 & 0.371 & \textbf{0.850} \\
 & OpenAI tool loop & 0.437 & $3$--$10$ & 0.172 & \textbf{0.665} & 0.544 & 0.126 & 0.425 & 0.689 \\
 & SurveyForge & 0.431 & $3$--$10$ & 0.220 & 0.529 & 0.439 & \underline{0.255} & 0.382 & 0.760 \\
 & SurveyG & 0.425 & $6$--$10$ & 0.155 & 0.487 & 0.486 & \textbf{0.284} & \underline{0.451} & 0.686 \\
 & AutoSurvey & 0.419 & $5$--$11$ & 0.238 & 0.540 & 0.444 & 0.117 & 0.411 & 0.766 \\
 & LLM$\times$MR & 0.379 & $11$--$12$ & 0.077 & 0.504 & 0.445 & 0.175 & 0.403 & 0.671 \\
\arrayrulecolor{black!25}\midrule\arrayrulecolor{black}
claude-sonnet-5 & Claude Code & 0.458 & $1$--$8$ & 0.227 & 0.573 & 0.520 & 0.188 & \textbf{0.467} & 0.775 \\
 & Claude Science & 0.456 & $1$--$8$ & 0.287 & \underline{0.598} & \underline{0.546} & 0.164 & 0.375 & 0.767 \\
\bottomrule
\end{tabular}
% END tab:leaderboard (b)
\end{table}

\paragraph{The leaderboard at full length.}
Table~\ref{tab:leaderboard-full} reads the leaderboard's systems at full report length, the
secondary and descriptive view; conventions as in Table~\ref{tab:leaderboard}. The top of the
fixed-input table is unchanged, while AutoSurvey --- a median of about $250{,}000$ words per task ---
rises to third, because the full-length composite rewards material that longer reports contain more
of. Under same pool the certified intervals widen and nothing is pinned.

\begin{table}[H]
\centering\footnotesize
\setlength{\tabcolsep}{3pt}
\caption{\textbf{The leaderboard at full length} (secondary, descriptive). Conventions as in
Tables~\ref{tab:leaderboard} and~\ref{tab:leaderboard-pool}; the footnotes of Table~\ref{tab:leaderboard}
apply to panel (a).}
\label{tab:leaderboard-full}
% BEGIN tab:leaderboard-full (a)
% generated by ccbench/quick/make_results_figs.py -- do not edit by hand
\setlength{\tabcolsep}{2.0pt}
\begin{tabular}{@{}p{58pt}p{76pt}Y{41pt}Y{24pt}Y{33.2pt}Y{33.2pt}Y{33.2pt}Y{33.2pt}Y{33.2pt}@{}}
\toprule
& & & & \multicolumn{5}{c}{Capability (fraction of human level)}\\
\cmidrule(l){5-9}
Backbone & System & Composite & Rank & \rot{Synthesis} & \rot{Reasoning} & \rot{Planning} & \rot{Writing} & \rot{Form} \\
\midrule
\multicolumn{9}{c}{\emph{(a) Fixed input: every system receives the review's reference list}}\\
\midrule
\rowcolor{oursrow} \cellcolor{white}Qwen3.8-27B & SCRIBE & \textbf{0.609} & $1$--$4$ & 0.695 & 0.596 & \textbf{0.319} & \textbf{0.619} & \textbf{0.816} \\
\rowcolor{oursrow}  & SCRIBE (untrained) & \underline{0.585} & $1$--$6$ & 0.690 & 0.603 & \underline{0.276} & 0.569 & \underline{0.787} \\
\arrayrulecolor{black!25}\midrule\arrayrulecolor{black}
\rowcolor{oursrow} \cellcolor{white}gpt-5.6-luna & SCRIBE-Luna & 0.573 & $2$--$8$ & 0.710 & 0.622 & 0.177 & 0.592 & 0.763 \\
 & AutoSurvey & 0.567 & $1$--$8$ & \textbf{0.958} & \textbf{0.769} & 0.135 & 0.370 & 0.605 \\
 & SurveyG & 0.530 & $6$--$12$ & 0.705 & 0.640 & 0.272 & 0.353 & 0.678 \\
 & SurveyGen-I & 0.520 & $6$--$13$ & 0.695 & 0.580 & 0.182 & 0.415 & 0.729 \\
 & OpenAI tool loop & 0.515 & $7$--$13$ & 0.741 & 0.552 & 0.136 & 0.487 & 0.660 \\
 & LiRA & 0.499 & $5$--$13$ & \underline{0.755} & 0.634 & 0.254 & 0.281 & 0.594 \\
 & LLM$\times$MR & 0.491 & $10$--$13$ & 0.749 & \underline{0.663} & 0.182 & 0.280 & 0.578 \\
\arrayrulecolor{black!25}\midrule\arrayrulecolor{black}
claude-sonnet-5 & Claude Science & 0.546 & $3$--$12$ & 0.702 & 0.604 & 0.160 & 0.526 & 0.736 \\
 & Claude Code & 0.535 & $5$--$12$ & 0.650 & 0.547 & 0.185 & 0.588 & 0.703 \\
\arrayrulecolor{black!25}\midrule\arrayrulecolor{black}
proprietary & Elicit SR$^{\S}$ & 0.555 & $2$--$11$ & 0.709 & 0.542 & 0.185 & \underline{0.610} & 0.731 \\
\bottomrule
\end{tabular}
% END tab:leaderboard-full (a)

\vspace{8pt}
% BEGIN tab:leaderboard-full (b)
% generated by ccbench/quick/make_results_figs.py -- do not edit by hand
\setlength{\tabcolsep}{2.0pt}
\begin{tabular}{@{}p{58pt}p{76pt}Y{41pt}Y{24pt}Y{27.0pt}Y{27.0pt}Y{27.0pt}Y{27.0pt}Y{27.0pt}Y{27.0pt}@{}}
\toprule
& & & & \multicolumn{6}{c}{Capability (fraction of human level)}\\
\cmidrule(l){5-10}
Backbone & System & Composite & Rank & \rot{Retrieval} & \rot{Synthesis} & \rot{Reasoning} & \rot{Planning} & \rot{Writing} & \rot{Form} \\
\midrule
\multicolumn{10}{c}{\emph{(b) Same pool: every system searches one frozen corpus}}\\
\midrule
\rowcolor{oursrow} \cellcolor{white}Qwen3.8-27B & SCRIBE & \textbf{0.506} & $1$--$8$ & \underline{0.335} & 0.570 & 0.574 & 0.250 & \textbf{0.468} & \textbf{0.841} \\
\rowcolor{oursrow}  & SCRIBE (untrained) & 0.477 & $1$--$10$ & \underline{0.335} & 0.579 & 0.578 & 0.232 & 0.353 & 0.787 \\
\arrayrulecolor{black!25}\midrule\arrayrulecolor{black}
gpt-5.6-luna & AutoSurvey & \underline{0.490} & $1$--$8$ & 0.238 & \textbf{0.906} & \textbf{0.746} & 0.117 & 0.257 & 0.677 \\
\rowcolor{oursrow}  & SCRIBE-Luna & 0.489 & $1$--$8$ & \underline{0.335} & 0.599 & 0.561 & 0.185 & 0.423 & 0.832 \\
 & SurveyForge & 0.476 & $1$--$8$ & 0.220 & \underline{0.744} & 0.618 & \underline{0.255} & 0.274 & 0.746 \\
 & SurveyGen-I & 0.473 & $1$--$9$ & \textbf{0.336} & 0.655 & 0.547 & 0.181 & 0.283 & \underline{0.835} \\
 & OpenAI tool loop & 0.438 & $6$--$11$ & 0.172 & 0.668 & 0.545 & 0.126 & 0.427 & 0.689 \\
 & SurveyG & 0.434 & $5$--$11$ & 0.155 & 0.621 & 0.597 & \textbf{0.284} & 0.278 & 0.670 \\
 & LLM$\times$MR & 0.404 & $9$--$12$ & 0.077 & 0.708 & \underline{0.635} & 0.175 & 0.249 & 0.579 \\
\arrayrulecolor{black!25}\midrule\arrayrulecolor{black}
claude-sonnet-5 & Claude Science & 0.455 & $1$--$12$ & 0.281 & 0.601 & 0.550 & 0.164 & 0.416 & 0.770 \\
 & Claude Code & 0.449 & $1$--$12$ & 0.222 & 0.573 & 0.520 & 0.188 & \underline{0.467} & 0.775 \\
\bottomrule
\end{tabular}
% END tab:leaderboard-full (b)
\end{table}

\paragraph{Commercial systems.}
Claude Science is Anthropic's research workbench and Claude Code is Anthropic's coding agent run
as a tool loop at low effort; both use claude-sonnet-5 and reach the corpus only through the tools we
provide. The OpenAI tool loop is gpt-5.6-luna called through the OpenAI Responses API with the same
tools. Under fixed input these tools return only the task's reference list. Elicit Systematic Review uses its own model, receives the same reference list, reads full text
where the other systems read abstracts, and cites at most $80$ papers; all of its citation marks
resolve. Those are properties of the system and sit inside the system window, so they are declared
rather than refused.

\section{Readout validity under evidence ablation}
\label{app:readout-validity}

The admission screen asks whether a readout carries topical signal: scoring a system against an
unrelated review must do worse than against its own. It does not ask whether a readout responds to
the \emph{evidence} a system was given. We test that on the controlled writing window, where six
published pipelines were each handed the same sealed entry (papers, outline and leaf allocation) on
$20$ tasks, one per subfield cluster. The same driver, AutoSurvey's writer, was then run with no papers, with another section's
papers, and with half of each section's papers.

Given the same driver, a writer with no papers scores \emph{higher} than one given them (Table~\ref{tab:ablation}) on form
($-0.114$ for with-evidence minus no-evidence, 95\% CI $[-0.170,-0.059]$), reasoning ($-0.061$,
$[-0.110,-0.013]$) and synthesis ($-0.044$, $[-0.087,-0.002]$), and lower on writing ($+0.071$,
$[+0.040,+0.102]$). Averaged over the six pipelines, the only readout that separates full evidence
from both the empty and the shuffled controls is evidence agreement ($+0.145$ and $+0.078$). The
content readouts therefore measure agreement with the human review's claims, relations and
organisation, which a strong backbone partly supplies from its own knowledge. That is the reading
the leaderboard adopts: its relations certify content agreement with the human review, the ground
truth of BioLitBench, and not the use of the evidence provided. An evidence-ablation control belongs
in any admission screen that intends the stronger reading.

\begin{table}[H]
\centering\small
\caption{\textbf{Evidence ablation at the controlled writing window} ($20$ tasks, matched length,
fraction-of-human-level scale). Each entry is the paired difference between the writer given the
full evidence and a control; the 95\% interval is below it. \emph{Same driver}: AutoSurvey's own
writer against its controls. \emph{Pipeline mean}: the six pipelines' mean against the same
controls. A negative value means the control scores higher.}
\label{tab:ablation}
\setlength{\tabcolsep}{5pt}
\begin{tabular}{@{}l ccc cc@{}}
\toprule
& \multicolumn{3}{c}{same driver} & \multicolumn{2}{c}{pipeline mean} \\
\cmidrule(lr){2-4}\cmidrule(lr){5-6}
Axis & no evidence & shuffled & half & no evidence & shuffled \\
\midrule
synthesis & $-$0.044 & $-$0.036 & $-$0.019 & $-$0.053 & $-$0.045 \\
 & {\scriptsize\color{black!60}[$-$.087, $-$.002]} & {\scriptsize\color{black!60}[$-$.075, $+$.003]} & {\scriptsize\color{black!60}[$-$.056, $+$.018]} & {\scriptsize\color{black!60}[$-$.088, $-$.018]} & {\scriptsize\color{black!60}[$-$.082, $-$.008]} \\[2pt]
reasoning & $-$0.061 & $-$0.035 & $-$0.027 & $-$0.081 & $-$0.055 \\
 & {\scriptsize\color{black!60}[$-$.110, $-$.013]} & {\scriptsize\color{black!60}[$-$.100, $+$.029]} & {\scriptsize\color{black!60}[$-$.080, $+$.027]} & {\scriptsize\color{black!60}[$-$.119, $-$.042]} & {\scriptsize\color{black!60}[$-$.109, $-$.000]} \\[2pt]
writing & $+$0.071 & $+$0.011 & $-$0.001 & $+$0.094 & $+$0.034 \\
 & {\scriptsize\color{black!60}[$+$.040, $+$.102]} & {\scriptsize\color{black!60}[$-$.027, $+$.048]} & {\scriptsize\color{black!60}[$-$.023, $+$.021]} & {\scriptsize\color{black!60}[$+$.067, $+$.122]} & {\scriptsize\color{black!60}[$+$.005, $+$.063]} \\[2pt]
form & $-$0.114 & $+$0.022 & $+$0.019 & $-$0.115 & $+$0.022 \\
 & {\scriptsize\color{black!60}[$-$.170, $-$.059]} & {\scriptsize\color{black!60}[$-$.024, $+$.068]} & {\scriptsize\color{black!60}[$+$.001, $+$.037]} & {\scriptsize\color{black!60}[$-$.170, $-$.059]} & {\scriptsize\color{black!60}[$-$.019, $+$.062]} \\[2pt]
\bottomrule
\end{tabular}
\end{table}

Scores do not depend significantly on the target review's publication year (pooled slope $-0.007$
per year, $p=0.15$--$0.17$), and the content axes do not rise with a review's age more than planning,
writing or retrieval ($p\ge0.34$). Publication year is only a proxy for a review's exposure during
pretraining.

\section{The capability campaign and the published pipelines}
\label{app:campaign}
\label{sec:campaign}

The leaderboard ranks systems by one composite. This appendix opens the same kind of comparison
capability by capability, on a smaller roster decided as its own family: the six fixed-input
systems of Table~\ref{tab:capability}, and seven same-pool systems in which SCRIBE runs its earlier
retrieval harness (Table~\ref{tab:capability-pool}); our system in this appendix is SCRIBE (untrained). Its rank intervals are therefore not those of
Tables~\ref{tab:leaderboard} and~\ref{tab:leaderboard-pool}.

\paragraph{Given a shared entry, the pipelines separate --- on some capabilities and not others.}
Under fixed input the ten admissible pairs across five axes give $50$ cells, of which $18$ certify:
five on synthesis, five on planning, three on writing, three on form and two on reasoning. There is
no retrieval axis here, because every arm was handed its evidence. The confirmatory family decides
$5$ of the $10$ pairs. Four involve SCRIBE, which is certified above each of its four comparable
opponents ($\widehat\Delta$ between $0.076$ and $0.110$ against $q$ between $0.037$ and $0.050$,
$n=50$); the fifth is SurveyG above LLM$\times$MR ($+0.035$ against $q=0.028$). Under same pool the seven-system roster gives $21$ pairs over six axes, $126$
cells, of which $32$ certify --- eleven on retrieval, eight on planning, eight on form and five on
synthesis. Reasoning and writing certify \emph{nothing at all}: once reports are held to a common
length, no two of these systems can be told apart on either, at this sample size.

\begin{table}[H]
\centering\small
\caption{\textbf{Capability scores as a fraction of human level} under \emph{fixed input}, where
every arm receives the task's reference bundle, with certified rank intervals in brackets --- the
positions consistent with every certified relation. LiRA receives the review's bibliography in its own input
format rather than through the shared allowlist, so this campaign does not pair it with the other arms,
and its whole row is unpinned. The
same-pool scores are in Table~\ref{tab:capability-pool}. Scores come from this campaign's own
per-axis scorer; SCRIBE (untrained)'s form score therefore reads $0.783$ here and $0.786$ on the merged board of
Table~\ref{tab:leaderboard}.}
\label{tab:capability}
\begin{tabular}{@{}lccccc@{}}
\toprule
\multicolumn{6}{@{}l}{\emph{fixed input} (reference bundle; no retrieval axis)}\\
System & synthesis & reasoning & planning & writing & form \\
\midrule
SCRIBE (untrained) & \textbf{0.686} [1,3] & \textbf{0.601} [1,4] & \textbf{0.276} [1,4] & \textbf{0.565} [1,4] & \textbf{0.783} [1,3]\\
AutoSurvey & 0.638 [1,5] & 0.533 [1,6] & 0.135 [3,6] & 0.483 [1,6] & 0.687 [2,6]\\
SurveyGen-I & 0.591 [2,5] & 0.510 [2,6] & 0.182 [3,6] & 0.449 [2,6] & 0.746 [1,6]\\
SurveyG & 0.568 [2,6] & 0.541 [1,6] & 0.272 [1,3] & 0.509 [1,5] & 0.689 [2,6]\\
LiRA & 0.539 [1,6] & 0.495 [1,6] & 0.254 [1,6] & 0.442 [1,6] & 0.716 [1,6]\\
LLM$\times$MR & 0.537 [4,6] & 0.515 [2,6] & 0.182 [2,6] & 0.446 [3,6] & 0.666 [2,6]\\
\bottomrule
\end{tabular}
\end{table}

\begin{table}[H]
\centering\small
\caption{\textbf{Capability scores under \emph{same pool}}, where every system retrieves for itself
from one frozen corpus and the retrieval axis becomes observable. Conventions as in
Table~\ref{tab:capability}. Writing, which certifies no pair, is not shown. The SCRIBE row uses its
earlier retrieval harness; the final harness,
used in the leaderboard, reaches retrieval $0.335$ (Table~\ref{tab:leaderboard-pool}).}
\label{tab:capability-pool}
\begin{tabular}{@{}lccccc@{}}
\toprule
System & synthesis & reasoning & planning & retrieval & form \\
\midrule
SurveyGen-I & 0.535 [1,6] & 0.485 [1,7] & 0.181 [3,6] & \textbf{0.336} [1,1] & \textbf{0.850} [1,4]\\
AutoSurvey & 0.540 [1,6] & 0.444 [1,7] & 0.117 [6,7] & 0.238 [2,3] & 0.766 [1,5]\\
SurveyForge & 0.529 [1,6] & 0.439 [1,7] & 0.255 [1,5] & 0.220 [2,4] & 0.760 [1,5]\\
SCRIBE (untrained, earlier harness) & 0.496 [1,6] & 0.480 [1,7] & 0.257 [1,6] & 0.156 [3,6] & 0.786 [1,6]\\
SurveyG & 0.487 [1,6] & 0.486 [1,7] & \textbf{0.284} [1,4] & 0.155 [4,5] & 0.686 [4,7]\\
DR-Tulu & 0.369 [6,7] & 0.419 [1,7] & 0.251 [1,6] & 0.084 [5,7] & 0.663 [3,7]\\
LLM$\times$MR & 0.504 [1,7] & 0.445 [1,7] & 0.175 [2,7] & 0.077 [6,7] & 0.671 [4,7]\\
\bottomrule
\end{tabular}
\end{table}

\paragraph{The two entry conditions diagnose a system rather than rank it.}
Reading Tables~\ref{tab:capability} and~\ref{tab:capability-pool} together is what the framework is
for. Given the same papers, SCRIBE is first or joint-first on every axis it is comparable on and no
baseline is certified above it anywhere; retrieving for itself with its earlier harness, its one
certified deficit was retrieval. The mechanism is in the artifacts: that harness recovered $50$
papers per task against $109$ under fixed input, its claim count fell from $118$ to $52$ and its
reports from $3{,}947$ to $2{,}644$ words. Everything downstream was working on a smaller bundle. A
single leaderboard number would have called the system mid-field and said nothing about where to
look; the per-window diagnosis located the deficit, and the final retrieval harness closed it
(retrieval $0.335$ against SurveyGen-I's $0.336$, Table~\ref{tab:leaderboard-pool}).

\paragraph{Refinement trades content for form.}
Scoring the draft and the final exit separately makes the refinement pass its own comparison, with
entries identical by construction. Five refinement effects certify and three are \emph{negative},
all on synthesis --- pipelines deleting human claims their own drafts had covered ($-0.072$,
$-0.113$, $-0.076$ against $q$ of $0.035$, $0.100$, $0.051$) --- against two positive effects, on
writing ($+0.068$) and form ($+0.115$). The same pass improves the prose and loses the content,
exactly the conflation that reading only a final report cannot detect. Across the panel, on the four axes with a draft exit
(all but planning), $10$ of the $13$ pair-by-axis relations certified at the final exit under fixed input are already present at the
draft exit, and $21$ of $24$ under same pool: most of what separates these systems is settled before refinement
begins.

\paragraph{A capability visible only at its own exit.}
Allocation --- which papers are handed to the writer for each section --- exists only at its own
exit. A human places a paper in $0.14$ of the reference's top-level sections; pipelines that
broadcast their whole evidence set to nearly every section sit between $0.94$ and $1.00$.
Broadcasting makes raw co-location recall $1.000$ by construction, carrying the broadcaster on one
task to an $F_1$ of $0.800$ against $0.286$ for a system genuinely allocating. Chance correction
removes that expectation and returns both to near zero, which is why the adjusted Rand index is the
only certifying readout here. Measured that way every system is far below human, SCRIBE at
$0.152$--$0.174$ against $0.099$ for the best baseline.

\paragraph{The published pipelines alone.}
Table~\ref{tab:certified} gives the certified tiers of the published pipelines with the backbone
held fixed, as discussed in Section~\ref{sec:framework-results}. Under a sign-flip null the
per-pair false-certification rate is $0.0035$--$0.0047$ against a nominal $0.005$, and leaving out
one subfield gives a mean Jaccard overlap of $0.994$ with no winner flips.

\begin{table}[H]
\centering\small
\caption{\textbf{The published pipelines alone} (scaffold family, $10$ pairs per tier; compare
Tables~\ref{tab:leaderboard} and~\ref{tab:leaderboard-pool}, whose larger families give wider intervals). \emph{Display rank} orders point estimates and is not
a claim; \emph{certified rank} is the interval of positions consistent with the certified
relations over every linear extension of their transitive closure, so a single value is a position
the data pin down. In the reference-fed tier, each pipeline runs its reference-fed entry. Composite is
the equal-weight mean over the tier's capability axes on the fraction-of-human-level scale, at
matched length. The last column gives the certified rank at full length.}
\label{tab:certified}
\begin{tabular}{@{}llccc@{}}
\toprule
& & & \multicolumn{2}{c}{Certified rank} \\
\cmidrule(l){4-5}
Tier & System & Composite & matched & full length \\
\midrule
self-retrieving & SurveyGen-I & 0.460 & $\mathbf{1}$ & $1$--$3$ \\
(certified: 7 / 7 of 10) & SurveyForge & 0.431 & $2$--$4$ & $1$--$3$ \\
 & SurveyG & 0.425 & $2$--$4$ & $\mathbf{4}$ \\
 & AutoSurvey & 0.419 & $2$--$4$ & $1$--$3$ \\
 & LLM$\times$MR & 0.379 & $\mathbf{5}$ & $\mathbf{5}$ \\
\midrule
reference-fed & SurveyG & 0.516 & $1$--$4$ & $2$--$4$ \\
(certified: 1 / 6 of 10) & SurveyGen-I & 0.496 & $1$--$5$ & $2$--$4$ \\
 & AutoSurvey & 0.495 & $1$--$5$ & $\mathbf{1}$ \\
 & LiRA & 0.484 & $1$--$5$ & $2$--$5$ \\
 & LLM$\times$MR & 0.469 & $2$--$5$ & $4$--$5$ \\
\bottomrule
\end{tabular}
\end{table}

\section{Theory checks in detail}
\label{app:theory-checks}

Table~\ref{tab:theory-checks} summarizes the tests of the theory. The subsections below give
each test in detail.

\begin{table}[H]
\centering\footnotesize
\setlength{\tabcolsep}{4pt}
\caption{\textbf{Verification summary.} Each part of the theory, what it is tested on, and the key
result. The subsection with the full test is given in parentheses.}
\label{tab:theory-checks}
\begin{tabularx}{\textwidth}{@{}>{\raggedright\arraybackslash}p{3.3cm} >{\raggedright\arraybackslash}p{3.2cm} L >{\raggedright\arraybackslash}p{2.3cm}@{}}
\toprule
Component & Tested on & Key result & Verdict \\
\midrule
\multicolumn{4}{@{}l}{\emph{On the published systems}}\\[1pt]
Attribution bound, Thm.~\ref{thm:observable-attribution} {\scriptsize(\ref{app:tc-two-entries})} & two entries per pipeline, $118$ runs & $L_M$ transfers to held-out pipelines (coverage $0.79$--$1.00$) & $\checkmark$ holds \\[2pt]
Bound at a stage window {\scriptsize(\ref{app:tc-w4})} & one sealed writing input, $20$ tasks; fit on $5$ pipelines, check on $15$ pairs of $6$ & $420/420$ cells hold at matched length; $44\%$ vacuous & $\checkmark$ holds; informative on writing \\[2pt]
Admissibility, Def.~\ref{def:fair-comparability} {\scriptsize(\ref{app:tc-refusals})} & five LiRA pairs, measured entry mismatch & radius $0.70$--$0.73$ against tolerance $0.21$; refused on $15/15$ readouts & $\checkmark$ refuses \\[2pt]
Non-expansive windows, Cor.~\ref{cor:nonexpansive} {\scriptsize(\ref{app:tc-two-entries}, \ref{app:tc-w4})} & both designs above & $0$ runs expand at matched length; $62/100$ at full length & $\checkmark$ at matched length \\[2pt]
Certification rule, Eq.~\ref{eq:decision} {\scriptsize(\ref{app:tc-calibration})} & sign-flip null; leave one subfield out & false certification $0.035$ / $0.062$; no flips & $\checkmark$ near nominal \\[2pt]
Readout validity {\scriptsize(\ref{app:readout-validity})} & evidence ablation, writing stage & no-evidence writer scores higher on synthesis, reasoning, form & $\sim$ content agreement \\[2pt]
Comparability vs.\ equivalence, Def.~\ref{def:equiv} {\scriptsize(\ref{app:tc-ctrl})} & $\Ctrl_M$ on reference-fed pairs & a fair pair: $0.07$ at retrieval, $0.51$ at writing & $\checkmark$ distinct \\
\midrule
\multicolumn{4}{@{}l}{\emph{Synthetic chains or implementation}}\\[1pt]
Composition, Thm.~\ref{thm:composition} {\scriptsize(\ref{app:tc-composition})} & synthetic four-stage chains & holds on $95$--$100\%$ of chains & $\checkmark$ synthetic only \\[2pt]
Fixed point, Thm.~\ref{thm:fixed-point} {\scriptsize(\ref{app:tc-fixed})} & empirical kernels & contraction $0.449\le0.45$; unique & $\circ$ implementation \\
\bottomrule
\end{tabularx}

\vspace{2pt}
\parbox{\textwidth}{\scriptsize $\checkmark$ holds as stated; $\sim$ holds with the scope given;
$\circ$ implementation check only. False certification: same pool / fixed input.}
\end{table}

\subsection{The attribution bound on two entries}
\label{app:tc-two-entries}

Each of four published pipelines ran twice on every task it completed under both entries: searching the frozen corpus, and with the
corpus restricted to the review's reference bundle. For each task, $x$ is the entry mismatch of the
self-retrieving run to the reference bundle, and $y$ is the distance between the two runs' writing
exits. This gives $118$ paired runs with $x\in[0.74,1.00]$. The selected estimator fits the
module-sensitivity constant as the $95$th percentile of $y/x$, giving $L_M=0.890$.
Table~\ref{tab:lmtests} reports the tests that could fail. The leave-one-pipeline-out test checks
transfer, because the panel applies one constant to every pipeline.
Non-expansiveness is falsifiable: the same estimator returns $1.74$ on the writing stage of the
\emph{development instrument}, an in-house four-stage agent used only as a measuring instrument,
whose stages can be entered from chosen inputs and so probed at smaller $x$. None of the $118$ runs
expands. Table~\ref{tab:designs} gives the estimator designs that were
compared; the fitted constant is a $95$th-percentile ratio, so the radius it gives is a
high-probability envelope rather than a worst-case bound.

\begin{table}[H]
\centering\small
\caption{\textbf{The module-sensitivity constant $L_M$, measured on the panel.} Fitted on the paired gap over $118$ paired runs from four pipelines, $x\in[0.74,1.00]$.}
\label{tab:designs}
\setlength{\tabcolsep}{6pt}
\begin{tabular}{@{}l cc l@{}}
\toprule
$L_M$ design & $L_M$ & $\xi$ & note \\
\midrule
\textbf{ratio $q_{95}$} & \textbf{0.890} & \textbf{0.000} & \textbf{selected design} \\
ratio sup & 0.967 & 0.000 & supremum form \\
affine    & 0.803 & 0.085 & envelope, slope $+$ slack \\
\midrule
\multicolumn{4}{@{}l@{}}{\emph{per system} (ratio $q_{95}$): AS $0.911$ ($n{=}19$), SGI $0.910$ ($19$), LMR $0.880$ ($31$), SG $0.863$ ($49$)}\\
\bottomrule
\end{tabular}
\end{table}

\begin{table}[H]
\centering\small
\caption{\textbf{Falsification tests for Assumption~\ref{ass:main}.} (a) A constant fitted on the other three pipelines, evaluated on the held-out one --- the transferability the panel-level radius actually assumes. (b) The non-expansive claim is falsifiable and is not falsified; the instrument's writing stage is a positive control.}
\label{tab:lmtests}
\setlength{\tabcolsep}{6pt}
\begin{tabular}{@{}l rrr r@{}}
\toprule
\multicolumn{5}{@{}l@{}}{\emph{(a) leave-one-system-out}}\\
held out & $n$ test & $L_M$ (train) & $L_M$ (own) & held-out coverage \\
\midrule
AutoSurvey   & 19 & 0.881 & 0.911 & 0.789 \\
SurveyGen-I  & 19 & 0.882 & 0.910 & 0.842 \\
LLM$\times$MR & 31 & 0.895 & 0.880 & 0.968 \\
SurveyG      & 49 & 0.908 & 0.863 & 1.000 \\
\midrule
\multicolumn{5}{@{}l@{}}{\emph{(b) falsifiability} \quad locally expansive runs ($y>x$): $0/118$ \quad max ratio $0.967$ \quad instrument writing stage: $1.74$}\\
\bottomrule
\end{tabular}
\end{table}

\subsection{The bound at a controlled writing window}
\label{app:tc-w4}

Six published pipelines were each handed the same sealed writing-window entry on $20$ tasks, one per
subfield cluster: the papers, the outline, and the assignment of papers to sections, built from the review's
own sections and citations. The drivers call each pipeline's own stage code. This exit is the
controlled contrast $\Delta^\star$; each pipeline's own writing exit on the same task gives
$\Delta^{\mathrm{obs}}$. Table~\ref{tab:w4bound} reports the fit and the score-level check of
Theorem~\ref{thm:observable-attribution} over the $30$ admitted readouts scored at the writing window and $15$ pairs. The $5$ pairs
with LiRA, which does not retrieve, carry no retrieval readouts, which leaves $420$ cells. The shared
entry's paper set is essentially the reference bundle, so the self-retrieving pipelines still sit at
$x\in[0.74,1.00]$; LiRA, which is reference-fed natively, sits at $x=0$ with a nonzero exit
displacement. A pure ratio cannot describe that case, so the fit uses the five self-retrieving
pipelines; the slack $\xi$ of the theorem is the term that covers it. The bound is loose: the median bound is $0.955$ against
a median gap of $0.060$. It is informative on the writing axis, where no cell is vacuous, and vacuous
on every retrieval cell, because retrieval readouts read the given entry rather than the writing
exit. Non-expansiveness holds at matched length but fails at full length, because report length
dominates the exit distance, so the corollary's hypothesis depends on the declared length convention.

\begin{table}[H]
\centering\small
\caption{\textbf{The attribution bound at a controlled writing window} ($5$ self-retrieving
pipelines $\times$ $20$ tasks for the fit; $420$ pair-by-readout cells for the score-level check).}
\label{tab:w4bound}
\begin{tabular}{@{}lcc@{}}
\toprule
& matched length & full length \\
\midrule
fitted $L_M$ ($95$th pct.\ / sup) & $0.906$ / $0.940$ & $1.042$ / $1.188$ \\
runs with $y>x$ & $0/100$ & $62/100$ \\
leave-one-pipeline-out coverage (mean / min) & $0.93$ / $0.80$ & $0.94$ / $0.90$ \\
bound holds (all cells) & $420/420$ & $419/420$ \\
vacuous cells ($\eout\ge1$) & $43.6\%$ & $46.0\%$ \\
median $\eout$ / median $|\Delta^{\mathrm{obs}}-\Delta^\star|$ & $0.955$ / $0.060$ & $0.978$ / $0.066$ \\
\bottomrule
\end{tabular}
\end{table}

\subsection{Refusals computed from the bound}
\label{app:tc-refusals}

Table~\ref{tab:refusals} applies Eq.~\ref{eq:outside-bound} to the five pairs of LiRA with a
self-retrieving pipeline. The common entry $\mu$ is the reference bundle, which is LiRA's own entry,
so $\epsilon_{\mathrm{in},\mathrm{LiRA}}=0.002$ and $\epsilon_{\mathrm{in},P}$ is the mean entry
mismatch of pipeline $P$'s self-retrieved papers to the bundle. With $L_C=1$ and no channel, protocol
or slack term (the selected ratio fit has $\xi=0$), Eq.~\ref{eq:outside-bound} reduces to
$\eoutn{n}=L_{s_n}L_M(\epsilon_{\mathrm{in},P}+\epsilon_{\mathrm{in},\mathrm{LiRA}})$, with
$L_M=0.890$ and $L_{s_n}$ the score constant of readout $n$ at the writing window (median $0.833$).
For AutoSurvey, $0.833\times0.890\times(0.975+0.002)=0.725$. The radius is compared with the
pre-registered tolerance of each readout, the median gap between human peer reviews. Of the $47$
readouts in the pre-registered tolerance table, $15$ have an admissible tolerance ($0<\epsilon_{\mathrm{fair}}<0.5$), with median $0.208$.
Every pair is refused on all $15$; the smallest margin is $+0.077$, and a readout would be admitted
only if $L_M\le0.554$, below every observed ratio. The verdicts are unchanged under the
leave-one-pipeline-out constants ($0.881$--$0.908$).

\begin{table}[H]
\centering\small
\caption{\textbf{Refusals computed from Eq.~\ref{eq:outside-bound}} for LiRA against each
self-retrieving pipeline.}
\label{tab:refusals}
\begin{tabular}{@{}lcccc@{}}
\toprule
Pair & $\epsilon_{\mathrm{in},P}$ & median $\eout$ & median $\epsilon_{\mathrm{fair}}$ & readouts refused \\
\midrule
AutoSurvey -- LiRA & $0.975$ & $0.725$ & $0.208$ & $15/15$ \\
SurveyForge -- LiRA & $0.968$ & $0.720$ & $0.208$ & $15/15$ \\
SurveyG -- LiRA & $0.972$ & $0.723$ & $0.208$ & $15/15$ \\
SurveyGen-I -- LiRA & $0.942$ & $0.700$ & $0.208$ & $15/15$ \\
LLM$\times$MR -- LiRA & $0.979$ & $0.728$ & $0.208$ & $15/15$ \\
\bottomrule
\end{tabular}
\end{table}

\subsection{Composition across stages}
\label{app:tc-composition}

Published pipelines cannot be re-entered at an intermediate artifact, so the multi-stage bound is
checked on synthetic four-stage chains built from the human graphs, with known knobs for retrieval
recall, claim retention, and section order (Table~\ref{tab:synthetic}). The bound holds everywhere
and is informative through the organisation stage. At the writing stage it exceeds one on
$[0,1]$-valued distances.

\begin{table}[H]
\centering\small
\caption{\textbf{Theorem~\ref{thm:composition} on synthetic four-stage chains built from the human graphs.} Observed chain discrepancy against the composed bound (affine at $q_{95}$, and ratio form); fraction of chains on which the bound holds; tightness $=$ observed/bound.}
\label{tab:synthetic}
\setlength{\tabcolsep}{6pt}
\begin{tabular}{@{}l c cc c c c@{}}
\toprule
Stage $m$ & observed $\W$ & bound (affine) & bound (ratio) & holds & non-vacuous & tightness \\
\midrule
1 retrieval    & 0.605 & 0.744 & 0.744 & 95\,\%  & yes & 0.79 \\
2 synthesis    & 0.347 & 0.844 & 0.891 & 100\,\% & yes & 0.38 \\
3 organization & 0.261 & 1.312 & 0.981 & 100\,\% & yes & 0.26 \\
4 writing      & 0.685 & 1.465 & 3.378 & 100\,\% & \textbf{no} & 0.21 \\
\bottomrule
\end{tabular}
\end{table}

\subsection{Calibration of the certification rule}
\label{app:tc-calibration}

We apply the leaderboard's decision rule to a null in which no system differs, on the ten-system version of each table that preceded the addition of the trained SCRIBE and Elicit. Each pair's paired
differences are centred, and their signs are flipped per subfield cluster, with one sign vector per
draw shared by every pair so that the family is defined. Table~\ref{tab:calibration} reports the
family-wise false-certification rate over $2{,}000$ draws, the radius scale $c^\star$ at which that
rate equals the nominal $0.05$, and the certified relations that $c^\star$ would remove. It also
reports leave-one-subfield-out stability over the $23$ clusters. The rebuilt tables reproduce the
published leaderboard exactly before any perturbation.

\begin{table}[H]
\centering\small
\caption{\textbf{Calibration and stability of the leaderboard's decision rule.} Family-wise rates
are over $2{,}000$ sign-flip draws (nominal $0.05$). $c^\star$: radius scale giving the nominal
rate. Jaccard: overlap of the certified set when one of the $23$ subfields is left out (mean /
minimum).}
\label{tab:calibration}
\setlength{\tabcolsep}{5pt}
\begin{tabular}{@{}ll ccc l cc@{}}
\toprule
Table & View & pairs & certified & family-wise & $c^\star$ & Jaccard & flips \\
\midrule
fixed input & matched length & $45$ & $21$ & $0.062$ & $1.027$ (drops $1$) & $0.952$ / $0.870$ & $0$ \\
same pool & matched length & $45$ & $20$ & $0.035$ & $0.966$ (drops $0$) & $0.918$ / $0.727$ & $0$ \\
fixed input & full length & $45$ & $20$ & $0.051$ & $1.003$ (drops $0$) & $0.965$ / $0.905$ & $0$ \\
same pool & full length & $45$ & $16$ & $0.041$ & $0.980$ (drops $0$) & $0.877$ / $0.750$ & $0$ \\
\bottomrule
\end{tabular}
\end{table}

In that version, the one relation that calibration would remove is Claude Code over AutoSurvey in the fixed-input
table of the leaderboard, whose gap exceeds its radius by a factor of only $1.02$. No relation involving our system is affected.

\subsection{Comparability versus equivalence}
\label{app:tc-ctrl}

Fair comparability constrains only what surrounds a window; equivalence constrains what happens
inside it. $\Ctrl_M$ (Definition~\ref{def:equiv}) measures the second. For task-matched pairs we use
the contextwise radius, whose coupling is the task diagonal, so
$\Ctrl_M(A,B)=\operatorname{mean}_t[\lambda_{\mathrm{in}}\,d(\mathrm{entry}_A,\mathrm{entry}_B)
+(1-\lambda_{\mathrm{in}})\,d(\mathrm{exit}_A,\mathrm{exit}_B)]$, at $\lambda_{\mathrm{in}}=0.5$.
Where a direct entry distance is not stored, the entry term is bracketed by $|e_A-e_B|$ and
$e_A+e_B$, using each system's distance to the reference, and both ends are reported.

Table~\ref{tab:ctrl} shows the case that Remark~\ref{rem:fairness-vs-equivalence} separates.
The reference-fed AutoSurvey and LiRA are a fair pair. At the retrieval window they share the entry exactly
and $\Ctrl_M=0.070$, so their retrieval windows are nearly equivalent. At the writing window the same
pair has $\Ctrl_M\approx0.51$: comparable, but not equivalent, since substituting one writer for the
other moves the exit by half the scale of the readout.

\begin{table}[H]
\centering\small
\caption{\textbf{$\mathrm{Ctrl}_M$ for the reference-fed pairs} at $\lambda_{\mathrm{in}}=0.5$ (Eq.~\ref{eq:partial-radius}); mean over the tasks both ran, with the essential-supremum form. Pairs with the reference-fed SurveyGen-I are omitted, because they were not recomputed after its re-run. Entry lo/hi bracket the entry term where only each system's distance to the reference is stored. At retrieval the entry is the shared task specification, so the entry term is $0$ and $\mathrm{Ctrl}_M$ measures behaviour inside the window alone.}
\label{tab:ctrl}
\setlength{\tabcolsep}{6pt}
\begin{tabular}{@{}l r cc c cc@{}}
\toprule
& & \multicolumn{2}{c}{retrieval} & & \multicolumn{2}{c}{writing} \\
\cmidrule(lr){3-4}\cmidrule(lr){6-7}
Pair & $n$ & $\mathrm{Ctrl}_M$ & sup & & $\mathrm{Ctrl}_M$ (lo--hi) & sup \\
\midrule
AutoSurvey -- LiRA        & 50 & 0.070 & 0.161 & & 0.504--0.506 & 0.606 \\
LiRA -- SurveyG           & 49 & 0.092 & 0.293 & & 0.509--0.511 & 0.749 \\
AutoSurvey -- SurveyG & 49 & 0.127 & 0.322 & & 0.392--0.505 & 0.901 \\
AutoSurvey -- LLM$\times$MR & 50 & 0.424 & 0.472 & & 0.800--0.940 & 1.116 \\
LLM$\times$MR -- SurveyG & 49 & 0.429 & 0.485 & & 0.757--0.941 & 1.232 \\
LiRA -- LLM$\times$MR     & 50 & 0.430 & 0.478 & & 0.852--0.854 & 0.971 \\
\bottomrule
\end{tabular}
\end{table}

\subsection{The controllability distance}
\label{app:tc-fixed}

The operator of Theorem~\ref{thm:fixed-point} is iterated with exact optimal transport on empirical
branch kernels of the logged systems (Table~\ref{tab:thm43}). The check confirms the
implementation: the contraction modulus, the unique fixed point, and the pseudometric properties.

\begin{table}[H]
\centering\small
\caption{Theorem~\ref{thm:fixed-point} and its corollaries on the empirical kernels of the twelve logged configurations.}
\label{tab:thm43}
\resizebox{\textwidth}{!}{%
\begin{tabular}{@{}l l l c@{}}
\toprule
Property & Requirement & Measured & \\
\midrule
contraction & ratio $\le\gamma(1-\beta)=0.45$ & $0.449$, $0.425$, $0.428$ (random inputs); $0.450$ at convergence & \checkmark \\
existence / uniqueness & same fixed point from $d_0\equiv0$ and $d_0\equiv1$ & 18 / 18 iterations, sup gap $5.7\times10^{-7}$ & \checkmark \\
pseudometric & symmetry, triangle inequality & symmetry gap $0$; $0$ / $20{,}000$ triangle violations & \checkmark \\
zero self-distance & $d^\dagger(S,S)=0$ & $0$ for all 12 systems & \checkmark \\
\bottomrule
\end{tabular}}
\end{table}

\section{Scope}
\label{app:scope}

\paragraph{A certificate compares systems up to an exit.}
A certified relation between system windows compares complete systems as built, backbone included
(Section~\ref{sec:method}). A scaffold comparison asks which scaffold (the prompts, tools and
control flow around the model) is better for a fixed model. It holds the backbone fixed, so what it
certifies is an ordering of scaffolds around one generator, and a pair on different backbones is
comparable only if that difference is accounted for. Comparability is indexed by the declared interface,
observation channel and readout family.

\paragraph{The tasks come from one distribution.}
The $50$ tasks are biomedical review questions drawn from one corpus. The subfield clustering makes
the radius account for dependence within that distribution; transfer beyond it is a separate question.

\paragraph{Stage windows of published pipelines cannot be re-entered.}
Theorem~\ref{thm:composition} needs each stage to be run from a controlled input. The published
pipelines log every model call, so their intermediate artifacts can be recovered and scored, but
they cannot be re-entered at an intermediate artifact of our choosing; only their system windows can
be perturbed (Corollary~\ref{cor:nonexpansive}). Pipelines that accept an intermediate artifact as
input, and record its hash, would make every stage testable.

\end{document}